\documentclass[conference]{IEEEtran}

\def\BibTeX{{\rm B\kern-.05em{\sc i\kern-.025em b}\kern-.08emT\kern-.1667em\lower.7ex\hbox{E}\kern-.125emX}}

\usepackage{nicefrac}
\usepackage{siunitx}
\usepackage{array,framed}
\usepackage{booktabs}
\usepackage{
  color,
  float,
  epsfig,
  wrapfig,
  graphics,
  graphicx,
  subcaption
}
\usepackage{textcomp,amssymb}
\usepackage{setspace}
\usepackage{latexsym,fancyhdr,url}
\usepackage{enumerate}
\usepackage[linesnumbered,ruled,vlined]{algorithm2e}
\usepackage{algpseudocode}
\usepackage{amsthm}
\usepackage{xparse} 
\usepackage{xspace}
\usepackage{thmtools}
\usepackage{multirow}
\usepackage{csvsimple}
\usepackage{balance}
\usepackage{multirow}
\SetKwInOut{Input}{input}
\SetKwInOut{Output}{output}
\usepackage{bbm}
\let\labelindent\relax
\usepackage{enumitem}
\usepackage{marvosym}
\setlist{leftmargin=4mm}
\graphicspath{{./images/}}

\newcommand{\partitle}[1]{\medskip \noindent \textbf{#1.}}

\usepackage{
  tikz,
  pgfplots,
  pgfplotstable
}
\usepackage{hyperref}

\usetikzlibrary{
  shapes.geometric,
  arrows,
  external,
  pgfplots.groupplots,
  matrix
}

\pgfplotsset{compat=1.9}

\usepackage{mathtools,}

\DeclareMathAlphabet{\mathcal}{OMS}{cmsy}{m}{n}

\DeclareGraphicsExtensions{%
    .png,.PNG,%
    .pdf,.PDF,%
    .jpg,.mps,.jpeg,.jbig2,.jb2,.JPG,.JPEG,.JBIG2,.JB2}

\newcolumntype{C}[1]{>{\centering\let\newline\\\arraybackslash\hspace{0pt}}m{#1}}
\newcolumntype{K}[1]{>{\centering\arraybackslash}p{#1}}
\newtheorem{assumption}{Assumption}
\newtheorem{lemma}{Lemma}
\newtheorem{theorem}{Theorem}

\newtheorem{definition}{Definition}
\newtheorem{remark}{Remark}

\makeatletter
\newsavebox\myboxA
\newsavebox\myboxB
\newlength\mylenA

\newcommand*\xbar[2][0.75]{%
    \sbox{\myboxA}{$\m@th#2$}%
    \setbox\myboxB\null
    \ht\myboxB=\ht\myboxA%
    \dp\myboxB=\dp\myboxA%
    \wd\myboxB=#1\wd\myboxA
    \sbox\myboxB{$\m@th\overline{\copy\myboxB}$}
    \setlength\mylenA{\the\wd\myboxA}
    \addtolength\mylenA{-\the\wd\myboxB}%
    \ifdim\wd\myboxB<\wd\myboxA%
        \rlap{\hskip 0.5\mylenA\usebox\myboxB}{\usebox\myboxA}%
    \else
        \hskip -0.5\mylenA\rlap{\usebox\myboxA}{\hskip 0.5\mylenA\usebox\myboxB}%
    \fi}
\makeatother

\definecolor{darkgreen}{rgb}{0.0, 0.5, 0.0}
\begin{document}

\title{Differentially Private Zeroth-Order Methods for Scalable Large Language Model Fine-tuning}





\author{
\IEEEauthorblockN{Zhihao Liu\textsuperscript{1,2}, Jian Lou\textsuperscript{1,2}, Wenjie Bao\textsuperscript{1,2}, Yuke Hu\textsuperscript{1,2}, Bo Li\textsuperscript{3}, Zhan Qin\textsuperscript{1,2,\Letter}, Kui Ren\textsuperscript{1,2}}
\IEEEauthorblockA{\textsuperscript{1}The State Key Laboratory of Blockchain and Data Security, Zhejiang University\\
\textsuperscript{2}Hangzhou High-Tech Zone (Binjiang) Institute of Blockchain and Data Security, \textsuperscript{3}University of Chicago\\
\{zhihao\_liu, jian.lou, wenjie\_bao, yukehu, qinzhan, kuiren\}@zju.edu.cn; bol@uchicago.edu}
\thanks{\textsuperscript{\Letter} Corresponding author.}
\thanks{\textsuperscript{*} Code is available at \url{https://github.com/Liuxiaohao6/DPZOPO}.}
}

\IEEEoverridecommandlockouts
\makeatletter\def\@IEEEpubidpullup{6.5\baselineskip}\makeatother

\maketitle
\begin{abstract}
Fine-tuning on task-specific datasets is a widely-embraced paradigm of harnessing the powerful capability of pretrained LLMs for various downstream tasks. Due to the popularity of LLMs fine-tuning and its accompanying privacy concerns, differentially private (DP) fine-tuning of pretrained LLMs has been widely used to safeguarding the privacy of task-specific datasets. Lying at the design core of DP LLM fine-tuning methods is the satisfactory tradeoff among privacy, utility, and scalability. Most existing methods build upon the seminal work of DP-SGD. Despite pushing the scalability of DP-SGD to its limit, DP-SGD-based fine-tuning methods are unfortunately limited by the inherent inefficiency of SGD.

In this paper, we investigate the potential of DP zeroth-order methods for LLM pretraining, which avoids the scalability bottleneck of SGD by approximating the gradient with the more efficient zeroth-order gradient. Rather than treating the zeroth-order method as a drop-in replacement for SGD, this paper presents a comprehensive study both theoretically and empirically. First, we propose the stagewise DP zeroth-order method (DP-ZOSO) that dynamically schedules key hyperparameters. This design is grounded on the synergy between DP random perturbation and the gradient approximation error of the zeroth-order method, and its effect on fine-tuning trajectory.
Second, we further enhance the scalability by reducing the trainable parameters that are identified by repurposing a data-free pruning technique requiring no additional data or extra privacy budget. We propose DP zeroth-order stagewise pruning method (DP-ZOPO) with several pruning strategies and undertake a comparative analysis of these strategies. Additionally, we identify and elucidate the superior pruning strategy.

We provide theoretical analysis for both proposed methods. We conduct extensive empirical analysis on both encoder-only masked language model and decoder-only autoregressive language model, achieving impressive results in terms of scalability and utility regardless of the class of tasks (compared with DPZero \cite{zhang2024dpzero}, DP-ZOPO improves $4.5\%$ on SST-5, $5.5\%$ on MNLI with RoBERTa-Large and 9.2\% on CB, 3.9\% on BoolQ with OPT-2.7b when $\epsilon=4$, demonstrates more significant enhancement in performance on more complicated tasks).

\end{abstract}

\section{Introduction}
\label{sec:intro}
Pretrained Large Language Models (LLMs), scaling up to unprecedented sizes, have demonstrated remarkable potential in their capabilities of understanding and processing natural languages with human-like proficiency \cite{radford2018improving} \cite{radford2019language}. This has prompted a rapid surge in demand to harness the power of pretrained LLMs, particularly open-sourced series like OPT \cite{zhang2022opt}, llama \cite{touvron2023llama} and GPT \cite{radford2019language}, to boost performance across a wide range of downstream tasks. Fine-tuning is one of the most fundamental and dominant approaches for adapting pretrained LLMs to specific downstream tasks, which has been proven effective in  \cite{ziegler2019fine} \cite{sun2019fine}. Fine-tuning starts with the publicly accessible checkpoint of the selected model and continues to train the model for several epochs based on the task-specific dataset (referred to as the fine-tuning dataset hereafter). While appearing a simple process at first glance, the sheer scale of nowadays LLMs introduces significant scalability issues for fine-tuning, e.g., incurring prohibitive memory footprint, which complicates the design of training methods even in nonprivate fine-tuning \cite{zhang2023llama}. 


It is widely recognized that a finetuned LLM can leak sensitive information \cite{carlini2019secret} \cite{carlini2021extracting} from its fine-tuning dataset, which is considered private and valuable in certain downstream application areas related to finance, healthcare \cite{carlini2021extracting}. Therefore, private fine-tuning of pretrained LLMs has become a pressing need due to the escalating privacy concerns associated with the growing popularity of LLMs \cite{yu2021differentially}. Differential privacy (DP) \cite{dwork2006differential} is a widely adopted mathematical framework that ensures a rigorous privacy protection guarantee by introducing calibrated random perturbations. A long-standing research theme in DP is the pursuit of an ideal tradeoff between privacy and utility, which arises because the random perturbation for privacy-preserving purpose will inevitably degrade the utility. DP LLM fine-tuning faces even intensified tension as it has to handle the tradeoff among privacy, utility, and scalability. 

Most of the existing DP LLM fine-tuning methods \cite{li2021large} \cite{he2022exploring} build upon the seminal work of Differentially Private SGD (DP-SGD) \cite{abadi2016deep}, which clips the gradient vector per training sample, injects random Gaussian noise after clipping, and tracks the privacy budget across iterations by the moments accountant technique. These DP-SGD-based methods mitigate the scalability issue in roughly two lines of effort. The first line \cite{liu2021differentially} \cite{bu2022automatic} \cite{he2022exploring} aims to reduce the computational cost of the per-sample clipping of the gradient vector that is known to incur heavy computational and memory burden. For instance, \cite{liu2021differentially} introduce group clipping to replace the per-sample clip, \cite{bun2014fingerprinting} propose book keep (BK) strategy. The second line \cite{feng2023tensor} incorporates DP-SGD with parameter-efficient fine-tuning techniques in LLMs, which limits the number of trainable parameters. The intuition is that the utility-privacy tradeoff degrades quickly with respect to the number of trainable parameters since the DP perturbations need to be injected into all dimensions of the gradient vector. For instance, \cite{yu2021differentially} \cite{li2021large} consider LoRA, adapters, and compacter, which either train two smaller factors to approximate the LLM parameter update or insert additional trainable modules while freezing the pretrained LLM parameters. 

Although the DP-SGD-based methods mentioned above have made the best effort to squeeze performance out of DP-SGD, the inefficiency inherent in SGD continues to hinder achieving a satisfactory ``privacy-utility-scalability'' tradeoff for LLM fine-tuning. That is, the gradient calculation requires caching all intermediate activations during the forward pass and gradients during the backward pass, which leads to prohibitive memory usage for LLM fine-tuning, i.e., consuming up to $12\times$ the memory usage required for inference \cite{malladi2023fine}. Driven by such limitation, recent nonprivate LLM fine-tuning methods turn to zeroth-order methods, which circumvent the scalability issues associated with gradient computations by approximating them using two inferences \cite{malladi2023fine} \cite{zhao2024second} \cite{zhang2024revisiting}. In nonprivate LLM fine-tuning, the effectiveness of zeroth-order methods has been validated through both theoretical and empirical studies, which suggests a promising direction for developing the DP LLM fine-tuning method.

In this paper, we strive to achieve a better ``privacy-utility-scalab-ility'' tradeoff for DP LLM fine-tuning, through the lens of zeroth-order methods that are previously under-explored in DP literature. We comprehensively investigate DP zeroth-order methods from both theoretical and empirical perspectives, rather than treating the zeroth-order method as a drop-in replacement for SGD. 

We further leverage pruning to enhance DP zeroth-order LLM fine-tuning where LLM is updated on a subset of parameters with zeroth-order optimizer. We are motivated by the insights drawn from the following three factors. Firstly, a large proportion of LLM's parameters are typically redundant, and fine-tuning only a subset of parameters can yield significant improvements in downstream tasks. Secondly, reducing the dimensionality of the parameters to be updated improves the accuracy of zeroth-order gradient estimation. Lastly, since only a subset of parameters is updated with private data, the scale of DP noise introduced will also be reduced. Moreover, we further design multiple pruning strategies to enhance model performance under private settings.

\partitle{Our contributions}
First, we focus on the synergy between the DP random perturbations and the gradient approximation error of zeroth-order estimation to the true gradient, across different training stages. Specifically, we propose to dynamically adjust the ZO scale parameter that controls the gradient approximation error and divide the DP LLM fine-tuning into stages. We propose DP-ZOSO, the first stagewise algorithm for DP-ZO fine-tuning to our best knowledge. In earlier stages, the gradient approximation error is controlled to be smaller, and together with a larger learning rate, it allows the LLM fine-tuning to approach the optimum more quickly. In the later stages, we need to deliberately increase the gradient approximation error, together with a decreased learning rate, offering a stabilization effect on the fine-tuning trajectory. Our theoretical analysis demonstrates that this stagewise DP-ZO fine-tuning strategy provides an improved convergence rate. 

Second, we are the first to propose stagewise DP zeroth-order method in conjunction with reduced trainable parameters (DP-ZOPO). Unlike existing works in DP-SGD, which introduce additional trainable modules to modify the LLM structure, or the LoRA-based method, which still entails a certain number of variables, we propose to initially identify key parameters within the given LLM. Subsequently, we treat these identified parameters as trainable while freezing the remainder. To identify the key parameters, we repurpose a data-free pruning method, which does not necessitate public data or incur additional privacy costs during the pruning stage. In conjunction with the stagewise fine-tuning process, we also propose several dynamic pruning strategies and point out the superior strategy, it is proved that implementing pruning can enhance optimization both theoretically and empirically.

Third, we conduct extensive empirical evaluations to corroborate the superiority of our proposed DP zeroth-order methods for LLM fine-tuning. We utilize four different open-source LLMs, including masked language model (RoBERTa-large) and autoregressive language model (OPT-1.3b, OPT-2.7b) for downstream tasks such as sentiment classification, natural language inference, topic classification (including datasets like SST-2, SST-5, SNLI, MNLI, and TREC), question answering (SQuAD), translation (Europarl) and summarization (XSum and CNN/DM). Our method exhibits strong scalability and performs well across all models and datasets. Our method has enhanced the model's performance greatly (improves $4.5\%$ on SST-5, $5.5\%$ on MNLI with RoBERTa-Large and 9.2\% on CB, 3.9\% on BoolQ with OPT-2.7b when $\epsilon=4$) compared with existing differentially private zeroth-order fine-tuning method. Similar results can be observed in the tasks of question-answering, translation and summarization tasks.

\begin{figure*}[htbp]
    \centering
    \includegraphics[scale=0.6]{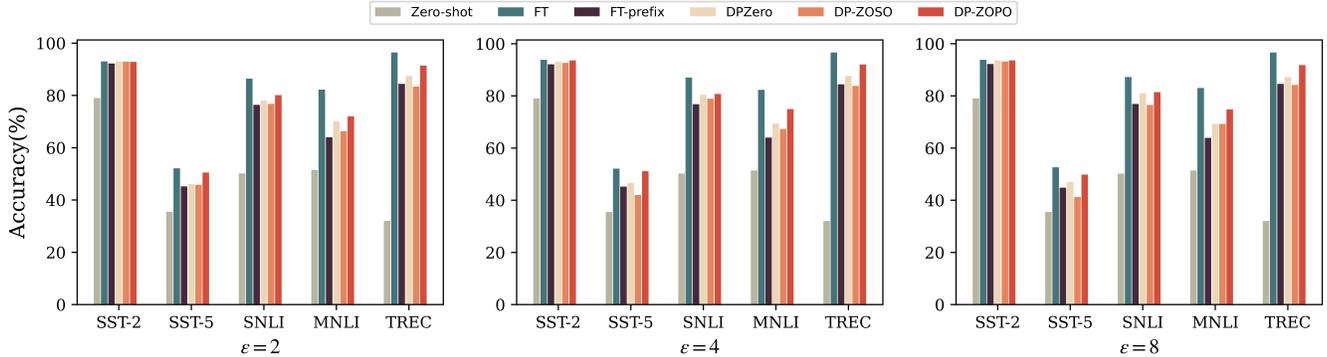}
    \caption{Experiments on RoBERTAa-large. We report zero-shot, DPZero \cite{zhang2024dpzero}, DP-ZOPO, DP-ZOSO and DP full-parameter fine-tuning (FT) and DP prefix-tuning (FT-prefix). DP-ZOSO and DP-ZOSO both outperform zero-shot, FT-prefix with much less memory. DP-ZOPO far outperforms DPZero, and DP-ZOSO and approaches FT  (Detailed numbers in Table \ref{table robert}).}\label{fig1}
\end{figure*} 
Our main contributions can be summarized as follows:
\begin{itemize}
    \item We are the first to conduct a comprehensive investigation of DP LLM fine-tuning from a zeroth-order perspective.
    \item We propose two novel stagewise DP zeroth-order algorithms DP-ZOSO and DP-ZOPO. DP-ZOSO optimizes the model stagewise on all dimensions while DP-ZOPO utilizes the pruning mask to guide the model towards updating in more important directions. We provide a comprehensive theoretical analysis of privacy and convergence for both of them. DP-ZOPO reduces the total complexity of DP-ZOSO by a factor of $1/r$ theoretically.
    \item We extensively experiment with our methods, providing empirical evidence of their scalability and utility. DP-ZOPO outperforms parameter-efficient method (DP prefix tuning) by 10.9\% on MNLI and 7.6\% on TREC when $\epsilon=4$ with RoBERTa-large. DP-ZOPO also outperforms DPZero by 5.5\% on MNLI and 4.4\% on TREC when $\epsilon=4$ with RoBERTa-large (see Figure \ref{fig1}). More results pertaining to a broader range of tasks are presented in Section \ref{sec generation task}.
\end{itemize}
In sum, we are the only work providing both theoretical and empirical studies among all concurrent works, offering more involved theoretical analysis and more comprehensive empirical analysis. Therefore, we believe our work offers distinct and significant contributions to both fields of DP LLM fine-tuning and DP zeroth-order methods, compared to these concurrent works.


\section{Background and Related Work}
\label{sec:relwork}

\subsection{Private Parameter Efficient Fine-Tuning} 
As LLMs scale up, full-parameter fine-tuning becomes impractical due to the requirement for extensive GPU memory. Things only get worse with privacy, which leads to overheads in terms of running time, memory usage, and most importantly, accuracy. The magnitude of noise introduced to a model due to DP grows as the model size increases, which can result in poor performance of LLMs. Parameter-efficient fine-tuning methods reduce memory consumption by updating just a fraction of the network parameters. Differentially private prompt tuning \cite{li2023privacy} \cite{duan2023flocks} has emerged as a simple and parameter-efficient solution for steering LLMs to downstream tasks. DP fine-tuning with Reparametrized Gradient Perturbation \cite{yu2021large} \cite{li2021large} computes the low-dimension projected gradient without computing the gradient itself, significantly reducing computation cost. Besides parameter-efficient fine-tuning, there also exist other memory-efficient methods.

\subsection{DP Zeroth-Order Fine-Tuning} 
Zeroth-order fine-tuning methods decrease memory consumption by replacing backpropagation with two inferences \cite{malladi2023fine} \cite{zhang2024revisiting}, each done with the same random perturbation with flipped signs. Recently, there have been several research that explores differential private zeroth-order fine-tuning: a workshop paper \cite{zhang2024dpzero} and a ``work in progress'' paper on arXiv \cite{tang2024private}. We stress that our work differs from them in three key aspects: 1) Regarding algorithm design, all these works merely consider constant scheduling of hyperparameters and have not explored parameter sparsification to further boost scalability like us; 2) Regarding theoretical analysis: \cite{tang2024private} did not provide any theoretical analysis for the utility. \cite{zhang2024dpzero}'s utility analysis did not consider the dynamic scheduling of the hyperparameters that render the analysis more involved; 3) Regarding empirical analysis: \cite{zhang2024dpzero} did not provide any empirical studies. While \cite{tang2024private} also considered LLM finetuing, our empirical study is more comprehensive by considering both RoBERTa-large and OPT models, along with serval complicated tasks. Moreover, targeting the sparsity inherent in LLMs, we propose DP-ZO fine-tuning with pruning and achieve better performance both theoretically and empirically.


\subsection{DP-SGD with Sparsification} 
Due to the low-rankness or the sparsity of neural networks, pruning has been widely used as a dimensionality reduction method to improve the scalability of DP-SGD \cite{adamczewski2023differential} \cite{luo2021scalable} \cite{mireshghallah2022differentially}. Adamczewski et. al. \cite{adamczewski2023pre} proposed parameter freezing pruning method: pre-prune the network and update those selected using DP-SGD. They \cite{adamczewski2023differential} further propose parameter selection: select which parameters to update at each step of training and update only those selected using DP-SGD. They use public data for freezing or selecting parameters to avoid privacy loss incurring in these steps. However, there has been no study about stagewise pruning method with dynamic pruning rates in differentially private zeroth-order fine-tuning.


\section{Preliminary}

\subsection{LLM Fine-Tuning}
LLM fine-tuning has been a popular way of adapting a pre-trained large language model to a specific task or domain by further training it on task-specific data.
\begin{definition}[LLM Fine-tuning]
    Fine-tuning a pretained language model $f(\boldsymbol{\theta})$ on the dataset $\mathcal{D}$ of downstream task can be described as the following optimization problem:
    \begin{equation}
        \min \limits_{\boldsymbol{\theta} \in \mathcal{R}^d} \{ f (\boldsymbol{\theta}, \mathcal{D})\}:= \frac{1}{n} \sum_{i=1}^{n} f(\boldsymbol{\theta}, x_{i}),
    \end{equation}
    where $x_{i}\in\mathcal{D}$ is the $i$-th training sample of the total training dataset $\mathcal{D}$ and $n$ is the number of training samples.
\end{definition}
Stochastic Gradient Descent (SGD) is an optimization algorithm commonly used in LLM fine-tuning. It's a variant of the gradient descent algorithm that's designed to handle large datasets more efficiently.
\begin{definition}[Stochastic gradient descent]
    SGD is a differentially private optimizer with learning rate $\eta$ that updates parameters as 
    \begin{equation}
        \boldsymbol{\theta}_{t}=\boldsymbol{\theta}_{t-1} - \eta\cdot\nabla f(\boldsymbol{\theta}_{t-1},\xi_{t}),
    \end{equation}
    where $\xi_{t}\in\mathcal{D}$ is the minibatch at time $t$ and $\nabla f(\boldsymbol{\theta}_{t-1},\xi_{t})$ is the average of gradients estimated by back propagation on $\xi_{t}$.
\end{definition}

\subsection{Differential Privacy}
Differential Privacy (DP) aims to provide a rigorous mathematical definition for privacy guarantees.
\begin{definition}[Differential Privacy \cite{dwork2006calibrating}]
    A randomized mechanism $\mathcal{A} : \mathcal{X} \to \mathcal{S}$ is called $(\epsilon,\delta)$ - differential private with respect to $\operatorname{d}$ if for every pair of adjacent datasets $X, X' \in \mathcal{X}$ satisfying $\operatorname{d}(X,X')\leq1$ and every subset of outputs $s \in \mathcal{S}$ it holds that:
    \begin{equation}
        \mathbb{P} \left[\mathcal{A}(X) \in s\right] \leq e^\epsilon * \mathbb{P} \left[\mathcal{A}(X') \in s\right] + \delta ,
    \end{equation}
    where $\operatorname{d}:\mathcal{X}^2 \rightarrow [0,\infty)$ be the distance between two datasets.
\end{definition}

DP-SGD is a privacy-preserving variant of SGD, designed to train with strong guarantees of differential privacy. 

\begin{definition}[DP-SGD \cite{abadi2016deep}]
    DP-SGD optimizes with clipping threshold $C$, noise scale $\sigma$, learning rate $\eta$ and updates parameters as 
    \begin{equation}
        \boldsymbol{\theta}_{t}=\boldsymbol{\theta}_{t-1} - \eta\cdot\frac{1}{m}\left(\sum_{i}\hat{g}(\boldsymbol{\theta}_{t-1},x_i)+N(0,\sigma^2 C^2\textbf{I}_d)\right),
    \end{equation}
    where $\hat{g}(\boldsymbol{\theta}_{t-1},x_i)$ is the gradient clipped by clipping threshold $C$, $\hat{g}(\boldsymbol{\theta}_{t-1},x_i)=g(\boldsymbol{\theta}_{t-1},x_i)/\max(1,\frac{||g(\boldsymbol{\theta}_{t-1},x_i)||_2}{C})$ and $g(\boldsymbol{\theta}_{t-1},x_i)$ denote the true gradient on data point $x_i$.
\end{definition} 

Moments accountant method tracks the accumulation of privacy loss over multiple iterations. By using the moments accountant, it is possible to estimate the overall privacy loss over a sequence of operations. 

\begin{lemma}[Moments Accountant (Poisson Subsampling) \cite{abadi2016deep}]\label{lemma ma}
    There exist constant $c_1$ and $c_2$ so that given the sampling probability $q=m/n$ and the number of steps $T$, for any $\epsilon< c_1q^2T$, DP-SGD is $(\epsilon,\delta)$-differentially private for any $\delta> 0$ if we choose
    \begin{equation}
        \sigma \geq c_2 \frac{q\sqrt{T\log(1/\delta)}}{\epsilon}.
    \end{equation}
\end{lemma}

\begin{lemma}[Privacy amplification by Poisson Subsampling and Uniform Sampling Without Replacement \cite{balle2018privacy}]\label{lemma poisson and uniform}
Due to the convenience of unifrom sampling in convergence analysis, uniform sampling is also adopted in \cite{bassily2021differentially} \cite{arora2023faster}.
\begin{itemize}
    \item \textbf{Poisson Subsampling.} 
    Let $\mathcal{A}:\mathcal{X} \rightarrow \mathcal{S}$ be a $(\epsilon,\delta)$-DP machanism. Then, the mechanism $\mathcal{A}^{\prime}: \mathcal{X}\rightarrow \mathcal{S}$ defined by $\mathcal{A}^{\prime}=\mathcal{A} \circ \mathrm{samp}_{\tau}^{\mathrm{po}}$ where the poisson subsampling mechanism $ \mathrm{samp}_{\tau}^{\mathrm{po}}$ takes a set $x$ and outputs a sample $y$ from the distribution $w=\mathrm{samp}_{\tau}^{\mathrm{po}}(x)$ supported on all set $y\subseteq x$ given by $w(y)=\tau^{|y|}(1-\tau)^{|x|-|y|}$. The mechanism $\mathcal{A}^{\prime}$ is $\left(\epsilon^{\prime}, \delta^{\prime}\right)$-DP, where 
    \begin{equation}
        \epsilon^{\prime}=\log \left(1+\tau\left(e^\epsilon-1\right)\right) , \delta^{\prime}=\tau \delta.
    \end{equation}
    \item \textbf{Uniform Sampling without Replacement.} 
    Let $\mathcal{A}:\mathcal{X}^{r_2} \rightarrow \mathcal{S}$ be a $(\epsilon,\delta)$-DP machanism. Then, the mechanism $\mathcal{A}^{\prime}: \mathcal{X}^{r_1} \rightarrow \mathcal{S}$ defined by $\mathcal{A}^{\prime}=\mathcal{A} \circ \mathrm{samp}_{r_2/r_1}^{\mathrm{wo}}$ where the subsampling mechanism $\mathrm{samp}_{r_2/r_1}^{\mathrm{wo}}$ takes a set $x$ of size $r_1$ and outputs a sample from the uniform distribution $w=\mathrm{samp}_{r_2/r_1}^{\mathrm{wo}}(x)$ over all sebsets $y\subseteq x$ of size $r_2\leq r_1$. The mechanism $\mathcal{A}^{\prime}$ is $\left(\epsilon^{\prime}, \delta^{\prime}\right)$-DP, where 
    \begin{equation}
        \epsilon^{\prime}=\log \left(1+\frac{r_2}{r_1}\left(e^\epsilon-1\right)\right) , \delta^{\prime}=\frac{r_2}{r_1} \delta.
    \end{equation}
    Since Poisson subsampling and uniform sampling without replacement share the same order of privacy amplification. Following the privacy analysis of Moments Accountant, DP-SGD with uniform data sampling without replacement shares the same privacy guarantee as DP-SGD with Poisson subsampling.
\end{itemize}
\end{lemma}

\subsection{Zeroth-Order Optimization}
Zeroth-order gradient estimation does not rely on explicit gradient information to update the parameters of a model. In zeroth-order methods, function evaluations at different points in the parameter space are used to approximate the gradient.
\begin{definition}[Zeroth-order gradient estimation \cite{spall1992multivariate}]
Given a model with parameters $\boldsymbol{\theta}\in\mathcal{R}^d$ and a loss function $f$, SPSA estimates the gradient on a minibatch $\xi$ as 
\begin{equation}
    \nabla f_{\beta}(\boldsymbol{\theta},\xi)=\frac{1}{2\beta}\left(f(\boldsymbol{\theta}+\beta\mathbf{v},\xi)-f(\boldsymbol{\theta}-\beta\mathbf{v},\xi)\right)\mathbf{v} ,
\end{equation}
where $\mathbf{v}\sim N(0,\textbf{I}_{d})$ and $\beta$ is the ZO scale parameter.
\end{definition}
$\textbf{Gaussian Smoothing:}$ Gaussian smoothing is a well-known technique that converts a possibly non-smooth function to a smooth approximation \cite{nesterov2017random}. Given a differentiable function $f:\mathcal{R}^d\rightarrow\mathcal{R}$ and $\beta\geq 0$, we define the Gaussian smoothing of $f$ as $f_{\beta}(\boldsymbol{\theta})=\mathbb{E}_{\mathbf{v}\sim N(0,\textbf{I}_{d})}\left[f(\boldsymbol{\theta}+\beta\mathbf{v})\right]$ where $N(0,\textbf{I}_{d})$ is a standard Gaussian distribution. Gaussian smoothing has the following properties.
\begin{lemma}\label{lemma f to smooth}
    Suppose $f:\mathcal{R}^d\sim\mathcal{R}$ is differentiable and $L$-Lipschitz. Then $\mathrm{(i)}$ $f_{\beta}$ is $L$-Lipschitz; $\mathrm{(ii)} ||f_{\beta}(\boldsymbol{\theta})-f(\boldsymbol{\theta})||\leq L\beta\sqrt{d}$; $\mathrm{(iii)} f_{\beta}$ is differentiable and $\frac{\sqrt{d}L}{\beta}$-smooth; $\mathrm{(iv)}$
    \begin{equation}
        \nabla f_{\beta}(\boldsymbol{\theta})=\mathbb{E}_{\mathbf{v}\sim N(0,\textbf{I}_{d})}\left[\frac{1}{2\beta}\left(f(\boldsymbol{\theta}+\beta\mathbf{v})-f(\boldsymbol{\theta}-\beta\mathbf{v})\right)\mathbf{v}\right].
    \end{equation}
\end{lemma}

Weakly-convex function is a class of “nice-behaviored" non-convex objectives that are close to a convex function which is well studied as non-convexity assumption in \cite{yuan2019stagewise} \cite{bassily2021differentially}.
\begin{definition}[Weakly-convex]
    The loss function $f(\boldsymbol{\theta})$ is $\rho$-weakly-convex if for every $\boldsymbol{\theta},\boldsymbol{\theta}'$, there exist constant $\rho$ that satisfies:
    \begin{equation}
        f(\boldsymbol{\theta}') \geq f(\boldsymbol{\theta}) + \langle\nabla f(\boldsymbol{\theta}),\boldsymbol{\theta}'-\boldsymbol{\theta}\rangle -\frac{\rho}{2}\left\|\boldsymbol{\theta}-\boldsymbol{\theta}'\right\|^2
    \end{equation}
\end{definition}
\begin{assumption}\label{assum 1}
    The function $f_{\beta}(\boldsymbol{\theta},x)$ is $L$-Lipschitz and $\frac{\sqrt{d}L}{\beta}$-smooth. There exist $\gamma$ that for every $x\in\mathcal{D}$, we have
    \begin{equation}
        \mathbb{E}\left[\left\|\nabla F_{\beta}(\boldsymbol{\theta})-\nabla f_{\beta}(\boldsymbol{\theta},x)\right\|^2\right] \leq \gamma^2.
    \end{equation}
\end{assumption}
It is notable that many papers have proposed and analyzed deterministic/stochastic optimization algorithms under the PL condition \cite{karimi2016linear} \cite{bassily2018exponential}.
\begin{assumption}\label{assum 2}
    The function $f_{\beta}(\boldsymbol{\theta},\mathcal{D})$ satisfies $\mu$-PL condition if for any $\boldsymbol{\theta}\in\mathcal{R}^d$ we have
    \begin{equation}
        \left\|\boldsymbol{\theta}-\boldsymbol{\theta}^*\right\|^2 \leq \frac{1}{2\mu}\left(f_{\beta}(\boldsymbol{\theta},\mathcal{D})-f_{\beta}(\boldsymbol{\theta}^*,\mathcal{D})\right).
    \end{equation}
\end{assumption}

\section{Methodology}
\label{sec:methodology}
In Section \ref{sec DPZOFT}, we introduce differential private zeroth-order fine-tuning and describe our novel DP noise injection method for zeroth-order. In Section \ref{sec dynamic}, we propose zeroth-order gradient estimation with dynamic ZO scale parameter. In Section \ref{sec stage}, we propose DP-ZOSO, the first stagewise DP-ZO fine-tuning method to our knowledge, and provide theoretical analysis for both privacy and convergence. In Section \ref{sec pruning}, we propose DP-ZOPO, the first stagewise DP-ZO fine-tuning method with data-free pruning to guide the model to update on important directions, whose schematic diagram is shown in Figure \ref{fig_pruning}.
\subsection{Differential Private Zeroth-Order Fine-Tuning}\label{sec DPZOFT}

The zeroth-order gradient is estimated by calculating the difference of the function on two points $\boldsymbol{\theta}+\beta\mathbf{v}$ and $\boldsymbol{\theta}- \beta\mathbf{v}$. Instead of clipping both the loss function $f(\boldsymbol{\theta}+\beta\mathbf{v})$ and $f(\boldsymbol{\theta}-\beta\mathbf{v})$, we clip the difference of the function on two points since the absolute value of the difference is much smaller than the loss function itself (3.2e-5 compared to 0.45 in average).

For privacy guarantee, Gaussian noise is injected into the clipped gradient during training. However, in the zeroth-order gradient descent algorithm, the direction of the approximate gradient $\nabla f_{\beta}(\boldsymbol{\theta})$, denoted by $\mathbf{v}$, is sampled from the standard Gaussian distribution independent of the private data. 

Compared to adding Gaussian to all dimensions of the gradient in DP-SGD, we propose a novel method of noise injection by adding noise $N(0,\sigma^2C^2)$ to $\frac{1}{2\beta}\left(f(\boldsymbol{\theta}+\beta\mathbf{v})-f(\boldsymbol{\theta}-\beta\mathbf{v})\right)$, achieving $(\epsilon,\delta)$-DP guarantee. By standard post processing, the approximate gradient $\nabla f_{\beta}(\boldsymbol{\theta})$ is also $(\epsilon,\delta)$-DP since the direction $\mathbf{v}$ is independent of private data. We also derive experimental results that with straightforward DP noise injection, the model fails to converge and exhibits output patterns almost indistinguishable from random noise.

\subsection{Zeroth-Order Gradient Estimation with Dynamic ZO scale parameter}\label{sec dynamic}

The zeroth-order method suffers from utility loss owing to the bias between the actual gradient and the zeroth-order gradient. Zeroth-order gradient is estimated by calculating the difference of the function on two points $\boldsymbol{\theta}+\beta\mathbf{v}$ and $\boldsymbol{\theta}- \beta\mathbf{v}$ where $\beta$ is the ZO scale parameter. The choice of ZO scale parameter $\beta$ has only been seen as a hyperparameter in \cite{malladi2023fine} \cite{ghadimi2013stochastic}. Shi et al. \cite{shi2022gradient} pointed out that the ZO scale parameter should be as small as possible but can be set too small in practice since the accuracy of the computing system is limited. However, when the model is close to convergence, the difference of the loss function on two points can be quite small such that the two-point estimated gradient appears disproportionately large when the ZO scale parameter is small. Thus, we propose dynamic zeroth-order gradient estimation with an increasing ZO scale parameter, reducing the bias of the zeroth-order gradient and offering a stabilization effect.

Furthermore, we find that zeroth-order gradients with bigger ZO scale parameters have the lower possibility of zeroth-order gradient getting clipped (detailed in Table \ref{table7}). More results of training loss with fix and dynamic ZO scale parameter are detailed in Table \ref{table dynamic scale}.




\begin{table}[tb]
\vspace{1.0em}
\caption{Clipping bias of RoBERTa-large fine-tuned on 1000 examples of SNLI with different ZO scale parameters. $\mathbb{P}(\mathrm{Clip})$ denotes the possibility of getting clipped.}
\centering
\begin{tabular}{c c c c c c}
\toprule
\specialrule{0em}{1pt}{1pt}
$\beta$ & 1e-6 & 2e-6 & 4e-6 & 6e-6 \\
\hline
\specialrule{0em}{1pt}{1pt}
\specialrule{0em}{1pt}{1pt}
$\mathbb{P}(\mathrm{Clip})$ & 13.40\% & 12.20\% & 11.70\% & 11.60\% \\
\bottomrule
\end{tabular}\label{table7}
\end{table}

\begin{figure}[tp]
    \flushleft
    \includegraphics[width=\linewidth]{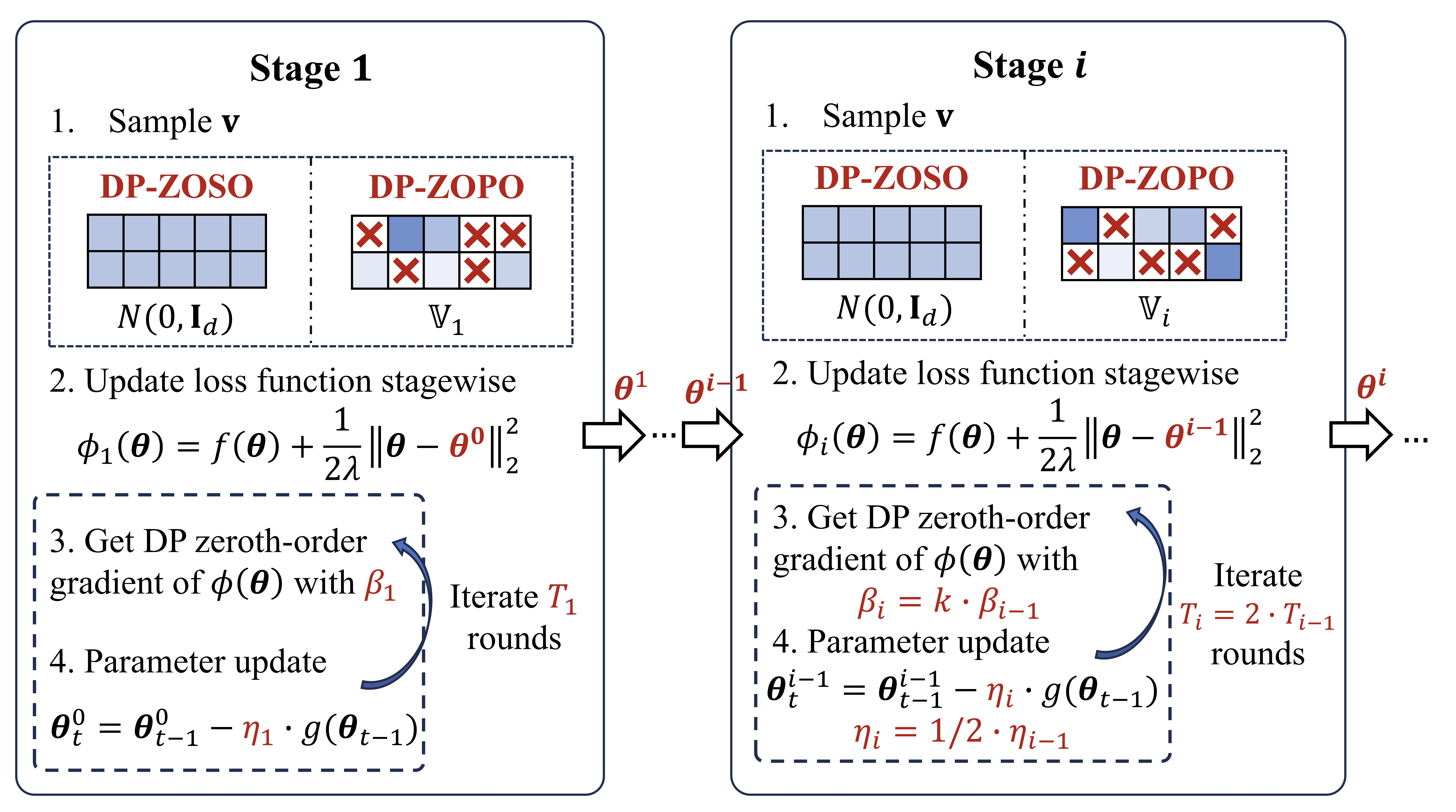}
    \caption{This figure illustrates the parameter variations in the stagewise algorithm and highlights the key difference between DP-ZOSO and DP-ZOPO, which lies in the sampling space of $\mathbf{v}$. The darker-colored dimensions represent those sampled from a normal distribution with larger standard deviation. The "X" denotes the value of the dimension is 0.}
    \label{fig_algorithm}
\end{figure}

\subsection{DP-ZO Stagewise Fine-Tuning}\label{sec stage}

First, let us propose algorithm \textbf{D}ifferentially \textbf{P}rivate \textbf{Z}eroth-\textbf{O}rder \textbf{S}tagewise \textbf{O}ptimizer (DP-ZOSO) in Algorithm \ref{alg:1}. We present the algorithm workflow in Figure \ref{fig_algorithm}. The training process is divided into $S$ stages. At the $s$-th stage, the regularization term $\frac{1}{2\lambda}\left\|\boldsymbol{\theta}-\boldsymbol{\theta}^{s-1}\right\|_2^2$ plays a crucial role in converting the weakly-convex loss function $f_{\beta_s}(\boldsymbol{\theta})$ to convex or strongly-convex $\phi_s(\boldsymbol{\theta})$, which divides the more difficult non-convex objective into easier and more tractable ones for each stage. The regularized function $\phi_s(\boldsymbol{\theta})$ is updated with the reference point $\boldsymbol{\theta}^{s-1}$ which is the returned solution from the previous stage. In earlier stages, the gradient approximation error is controlled to be smaller with smaller ZO scale parameter, and together with a larger learning rate, it allows the LLM fine-tuning to approach the optimum more quickly. In the later stages, we need to deliberately increase the gradient approximation error, together with a decreased learning rate, offering a stabilization effect on the fine-tuning trajectory.

\begin{algorithm}[htb]
    \SetAlgoLined
    \caption{DP-ZOSO (\textbf{D}ifferentially \textbf{P}rivate \textbf{Z}eroth-\textbf{O}rder \textbf{S}tagewise \textbf{O}ptimizer)}
    \label{alg:1}
    \Input{Initial point $\boldsymbol{\theta}^{0}\in\mathcal{R}^d$, initial stepsize $\eta_0\ge 0$, initial ZO scale parameter $\beta_0$, regularization parameter $\lambda$, initial iteration number $T_0$,  number of stage $S$, dataset $\mathcal{D}$}
    \For{$s=1$ to $S$}
    {
    $\beta_s=k\cdot\beta_{s-1},\ T_{s}=2\cdot T_{s-1},\ \eta_s=\eta_{s-1}/2$\;
    $\phi_s(\boldsymbol{\theta})=f_{\beta_s}(\boldsymbol{\theta})+\frac{1}{2\lambda}\left\|\boldsymbol{\theta}-\boldsymbol{\theta}^{s-1}\right\|_2^2$\;
    \small$\boldsymbol{\theta}^{s} = \textbf{DP-ZOO}(\phi_s(\boldsymbol{\theta}),\boldsymbol{\theta}^{s-1},\beta_s,T_s,\eta_s,N(0,\textbf{I}_{d}),\mathcal{D})$
    }
    \textbf{return} final model parameter $\boldsymbol{\theta}^{S}$
\end{algorithm}

Algorithm \ref{alg:2} is a \textbf{D}ifferentially \textbf{P}rivate \textbf{Z}eroth-\textbf{O}rder \textbf{O}ptimizer (DP-ZOO) that updates parameters in specific distribution $\mathbb{V}$. In the step $t$ of Algorithm \ref{alg:2}, $\xi_t$ is uniformly sampled without replacement from dataset $\mathcal{D}$. Random $P$ vectors $\left\{\mathbf{v}_t^p\right\}_{p=1}^{P}\in\mathbb{V}$ are sampled from distribution $\mathbb{V}$. For vector $\mathbf{v}_{t}^p$, we estimate the gradient step $\mathrm{gra}_{f,t}^{i,p}$ of original objective $f(\boldsymbol{\theta}_{t-1})$ on every data point $x_t^i$ in line 7 and then clip it to $\hat{\mathrm{gra}}_{f,t}^{i,p}$ in L2-norm. Gaussian noise is added to the sum of $\hat{\mathrm{gra}}_{f,t}^{i,p}$ and then average. Timing the direction $\mathbf{v}_{t}^p$ is the DP zeroth-order gradient $g_p(\boldsymbol{\theta}_{t-1})$ on direction $\mathbf{v}_{t}^p$. The gradient of the quadratic regularizer is added to the noisy clipped gradient in line 11 (no need to add DP noise). The gradient in iteration $t$ is $g(\boldsymbol{\theta}_{t-1})$ by taking average over $P$ directions. Parameters are then updated by $g(\boldsymbol{\theta}_{t-1})$ with learning rate $\eta_s$. We choose to add DP on the clipped ZO gradients of the original loss $\hat{\mathrm{gra}}_{f,t}^{i,p}$, because the varying gradient of the regularizer $\frac{1}{\lambda}(\boldsymbol{\theta}_{t-1}-\boldsymbol{\theta}_{0})$ does not access private data which will not cause privacy leakage and will complicate the clipping threshold setting.

\begin{algorithm}[tb]
    \SetAlgoLined
    \caption{\small DP-ZOO$(\phi_s(\boldsymbol{\theta}),\boldsymbol{\theta}^{s-1},\beta_s,T_s,\eta_s,\mathbb{V},\mathcal{D})$}
    \label{alg:2}
    \Input{stepsize $\eta_s\ge 0$, ZO scale parameter $\beta_s$, parameter dimension $d$, sample dataset $\xi_t$, sample size $m$, clipping threshold $C$, vector distribution $\mathbb{V}$ and iteration number $T_s$}
    Set initial parameter $\boldsymbol{\theta}_0=\boldsymbol{\theta}^{s-1}$\;
    \For{$t=1$ to $T_s$}
    {
    Randomly sample $\xi_t=\left\{x_{t}^{i}\right\}_{i=1}^{m}$ from dataset $\mathcal{D}$\;
    Sample $P$ random vectors $\footnotesize\left\{\mathbf{v}_t^p\right\}_{p=1}^{P}\in\mathbb{V}$\;
    \For{$p=1$ to $P$}
    {
    \For{$i=1$ to $m$}
    {
    $\scriptsize\mathrm{gra}_{f,t}^{i,p}=\frac{1}{2\beta_{s}}\left(f(\boldsymbol{\theta}_{t-1}+\beta_{s}\mathbf{v}_{t}^p,x_t^i)-f(\boldsymbol{\theta}_{t-1}- \beta_{s}\mathbf{v}_{t}^p,x_t^i)\right)$ \;
    $ \mathrm{Clip}: \hat{\mathrm{gra}}_{f,t}^{i,p}\gets\mathrm{gra}_{f,t}^{i,p}/\max\left\{1,\frac{|\mathrm{gra}_{f,t}^{i,p}|}{C}\right\}$\;
    }
    $\small g_p(\boldsymbol{\theta}_{t-1})=\frac{1}{m}\left(\sum_{i=1}^{m}\hat{\mathrm{gra}}_{f,t}^{i,p}+N(0,\sigma^2 C^2)\right)\cdot\mathbf{v}_{t}^p$\;
    $g_p(\boldsymbol{\theta}_{t-1})=g_p(\boldsymbol{\theta}_{t-1})+\frac{1}{\lambda}\left(\boldsymbol{\theta}_{t-1}-\boldsymbol{\theta}_0\right)$
    }
    $g(\boldsymbol{\theta}_{t-1})=\frac{1}{P}\sum_{p=1}^{P}g_p(\boldsymbol{\theta}_{t-1})$\;
    $\boldsymbol{\theta}_{t}=\boldsymbol{\theta}_{t-1}-\eta_s\cdot g(\boldsymbol{\theta}_{t-1})$\;
    }
    \textbf{return} $\boldsymbol{\theta}_{T_s}$
\end{algorithm}

\begin{theorem}[Privacy analysis of DP-ZOSO]
    In every stage of Algorithm \ref{alg:1} and in iteration $t$ of DP-ZOO, a random set $\xi_t$ of $m$ samples are sampled out of dataset $\mathcal{D}$ and $P$ random vectors are sampled from standard Gaussian distribution. Gaussian noise is added to the sum of zeroth-order on dataset $\xi_t$ in direction $\mathbf{v}_t^p$ after clipping $\left|\mathrm{gra}_{f,t}^{i,p}\right|$ to $C$. There exist constants $c_1$ and $c_2$, for any $\epsilon< c_1m^2T/n^2$, the overall algorithm is $(\epsilon,\delta)$-DP if
    \begin{equation}
        \sigma^2\geq\frac{ c_2^2 P^2m^2 T\log(P/\delta)}{\epsilon^2n^2}.\label{theo privacy}
    \end{equation}
\end{theorem}
\begin{proof}[Proof of Theorem \ref{theo privacy}]
    In DP-ZOSO, the sensitivity of every ZO estimated $\hat{\mathrm{gra}}_{f,t}^{i,p} (i\in[1,m])$ is clipped to $C$. The zeroth-order gradient on data point $x_t^i$ is estimated in $P$ random directions such that the privacy budget is divided equally into $P$. Since vectors are randomly sampled from standard Gaussian distribution, thus by post-processing, it suffices to show that the privacy of the term $\frac{1}{m}\sum_{i=1}^{m}\hat{\mathrm{gra}}_{f,t}^{i,p}$ is guaranteed. By applying Lemma \ref{lemma ma} and Lemma \ref{lemma poisson and uniform}, DP-ZOSO satisfies $(\epsilon,\delta)$-DP if $\sigma^2\geq\frac{ c_2^2 P^2m^2 T\log(P/\delta)}{\epsilon^2n^2}$.
\end{proof}

Before we analyze the convergence of DP-ZOSO, we will first present the approximation bound for noisy ZO gradients in the following Lemma \ref{lemma var}.
\begin{lemma}\label{lemma var}
    In DP-ZOSO, under assumption \ref{assum 1} and the dimension of the parameters to be updated is $d$ with clipping bound $C$. The error between the gradient $\nabla F_{\beta}(\boldsymbol{\theta})$ on the total dataset and the actual update gradient $\nabla\hat{f}_{\beta}(\boldsymbol{\theta}_t,\xi_{t+1})=\frac{1}{P}\sum_{p=1}^{P}\frac{1}{m}\sum_{i=1}^{m}\mathrm{CLIP}\left(\nabla f_{\beta_s}^{p}(\boldsymbol{\theta}_t,x_{t+1}^{i})\right)+\mathbf{z}_t^p$ is bounded by four terms:
    \begin{equation}
        \begin{split}
            \mathbb{E}[||\nabla& F_{\beta}(\boldsymbol{\theta}_{t})-\nabla\hat{f}_{\beta}(\boldsymbol{\theta}_t,\xi_{t+1})||^2] \\
            \leq&\frac{dc_2^2C^2 P T\log(P/\delta)}{\epsilon^2n^2}+\frac{8dC^2}{ePm}+\frac{64d\beta^2L^4}{Pm}+\frac{\gamma^2}{Pm},
        \end{split}
    \end{equation}
    where the first two terms are due to DP noise and clipping bias and the last two are due to zeroth-order estimation and data sampling errors.
\end{lemma}

Lemma \ref{lemma conver} presents the utility of DP-ZOO which is used in Theorem \ref{theo stage} to derive the total ZO oracle complexity.
\begin{lemma}\label{lemma conver}
    In DP-ZOSO, under assumption \ref{assum 1} and $f(\boldsymbol{\theta},x)$ is a $\rho$-weakly-convex function of $\boldsymbol{\theta}$, by applying DP-ZOO (Algorithm \ref{alg:2}) with $\eta_s\leq\frac{\beta_s}{\sqrt{d}L}$, for any $\boldsymbol{\theta}\in\mathcal{R}^d$, we have
    \begin{equation}
    \begin{split}
        \mathbb{E}&[\phi_{s}(\boldsymbol{\theta}_{T_s})-\phi_{s}(\boldsymbol{\theta})]
        \leq\left(\frac{1}{2\eta_s T_s}+\frac{1}{2T_s\lambda}\right)\left\|\boldsymbol{\theta}_{0}-\boldsymbol{\theta}\right\|^2 \\
        &+\eta_s \left(\frac{dc_2^2C^2PT\log(P/\delta)}{\epsilon^2 n^2}+\frac{64d\beta_s^2L^4}{Pm}+\frac{\gamma^2}{Pm}+\frac{8dC^2}{ePm}\right).\\
    \end{split}
    \end{equation}
\end{lemma}

\begin{theorem}\label{theo stage}
    In DP-ZOSO, suppose assumption \ref{assum 1} holds and the loss function $f(\boldsymbol{\theta},x)$ is $\rho$-weakly-convex of $\boldsymbol{\theta}$. Then by setting $\eta_s=\alpha_{s}\cdot\min\left\{\frac{Pm}{7\gamma^2},\frac{Pme}{56dC^2},\frac{\epsilon^2 n^2}{7dc_2^2C^2PT\log(P/\delta)},\frac{Pm}{448d\beta_S^2L^4}\right\}$ and $\lambda=\frac{7}{2\mu}, \eta_s T_s=\frac{7}{2\mu}$, after $S=\lceil\log(\alpha_0/\alpha)\rceil$ stages, we have
    \begin{equation}
        \phi_{s}(\boldsymbol{\theta}^{S})-\phi_{s}(\boldsymbol{\theta}^*)\leq \alpha.
    \end{equation}
    The total ZO oracle complexity of DP-ZOSO is 
    \begin{equation}
        \mathcal{O}\left(\left(\gamma^2+dC^2+d\beta_S^2L^4+\frac{dC^2P^2m\log(P/\delta)}{\epsilon^2 n^2}\right)\cdot\frac{1}{\mu\alpha}\right).
    \end{equation}
\end{theorem}

\subsection{DP-ZO Stagewise Fine-Tuning with pruning}\label{sec pruning}
The success in parameter-efficient fine-tuning motivates us that fine-tuning a subset of parameters LLMs can significantly improve the performance on downstream tasks. Pruning helps the model to update on more important directions, especially in the setting of differentially private zeroth-order fine-tuning. Compared with first-order optimizer, zeroth-order optimizer in conjunction with pruning yields greater enhancement in private settings (more details in Section \ref{sec_zo-fo pruningg}). 

We propose \textbf{D}ifferentially \textbf{P}rivate Stagewise \textbf{P}runing \textbf{O}ptimizer (DP-ZOPO) in Algorithm \ref{alg:5}. In DP-ZOPO, data-free pruning is employed to guide the model to optimize on more important directions sampled from $\mathbb{V}$. Pruning mask can be calculated at initial and remains static during the whole fine-tuning process, or can be updated during stages with constant or dynamic (increasing) pruning rate. Furthermore, we also propose incremental pruning strategy which is a more conservative pruning strategy.

We define the zeroth-order gradient estimation with pruning $\nabla_\mathbb{V} f_\beta(\boldsymbol{\theta})\overset{\text{def}}{=}\mathbb{E}_{\mathbf{v}\sim \mathbb{V}}\left[\frac{1}{2\beta}\left(f(\boldsymbol{\theta}+\beta\mathbf{v})-f(\boldsymbol{\theta}-\beta\mathbf{v})\right)\mathbf{v}\right]$ which samples vectors $\mathbf{v}$ from distribution $\mathbb{V}$.

\begin{figure}[t]
    \flushleft
    \includegraphics[scale=0.47]{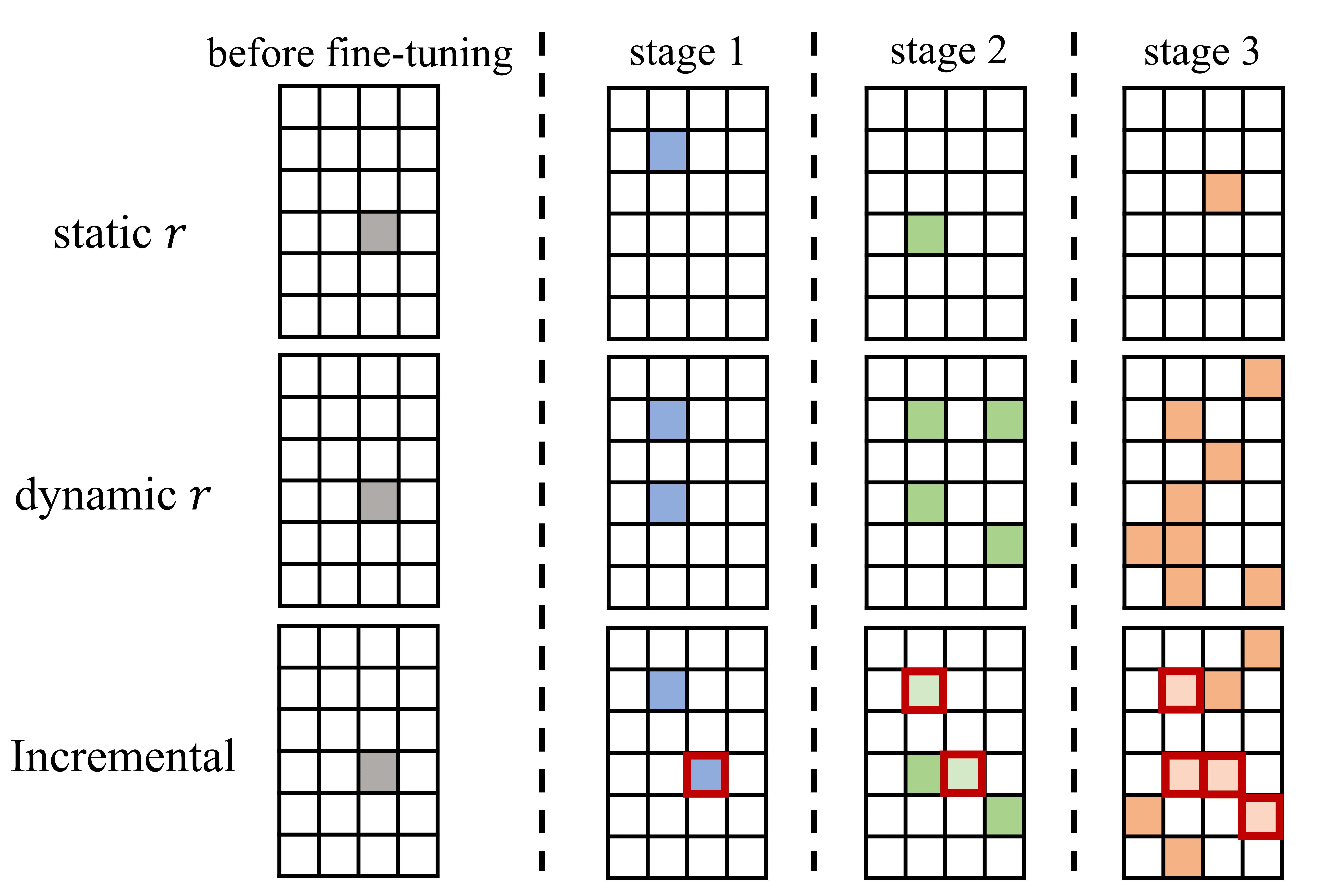}
    \caption{Different pruning strategies of dynamic pruning. The colored squares represent the parameters that require fine-tuning in the current stage. In incremental strategy, the highlighted squares indicate the parameters that were fine-tuned in the previous stage and will be trained in the current stage.}\label{fig_pruning}
\end{figure} 

\subsubsection{Constant Pruning Mask}
Algorithm \ref{alg:3} is a differentially private ZO fine-tuning framework with constant pruning mask calculated at initial. In Algorithm \ref{alg:3}, a new loss function $\mathcal{L}(\boldsymbol{\theta})$ is defined, where $\mathbbm{1}$ is the all ones vector and $|\boldsymbol{\theta}^{[l]}|$ is the element-wise absolute value of parameters in the $l$-th layer. Synaptic saliency scores $\nabla \mathcal{L}_{p}\odot\boldsymbol{\theta}$ are estimated. The scores will be used in Algorithm \ref{alg:4} to update the vector distribution which helps with convergence.
\begin{algorithm}[htb]
    \SetAlgoLined
    \caption{Data-free ZO pruning($\boldsymbol{\theta},r,type$)}
    \label{alg:3}
    \Input{Model parameter $\boldsymbol{\theta}$, pruning rate $r$ and matrix property $type$}
    Define $\mathcal{L}(\boldsymbol{\theta})\gets \mathbbm{1}^{\top}\left(\prod_{l=1}^{L}|\boldsymbol{\theta}^{[l]}|\right)\mathbbm{1}$\;
    Initialize binary mask: $mask=\mathbbm{1}$\;
    Sample $P$ random vectors $\left\{\mathbf{v}_p\right\}_{p=1}^{P}\in N(0,\textbf{I}_d)$ \;
    \For{$p=1$ to $P$}
    {
    $\nabla \mathcal{L}_{p}=\frac{1}{2\beta}\left(\mathcal{L}(\boldsymbol{\theta}+\beta\mathbf{v}_p)-\mathcal{L}(\boldsymbol{\theta}-\beta\mathbf{v}_p)\right)\mathbf{v}_p$
    }
    $Score\gets\frac{1}{P}\sum_{p=1}^{P}\nabla \mathcal{L}_{p}\odot\boldsymbol{\theta}$\;
    \textbf{Update vector distribution $\mathbb{V}$} \;
    $\mathbb{V}\gets \textbf{Importance Matrix}(Score,r,type)\cdot N(0,\mathbf{I}_d)$
\end{algorithm}

Parameter-efficient fine-tuning (PEFT) has been widely used in DP fine-tuning, which only updates a fraction of full parameters and can yield comparable performance. In DP fine-tuning, data-free pruning will not cause extra privacy concerns since no private data is used. Pruning freezes most of the pre-trained parameters and modifies the pre-trained model with small trainable parameters. 

The pruning mask is a sparse vector due to the pruning rate being set at around 1\%. To avoid additional GPU memory overhead, we use $\textbf{to\_sparse()}$ in NumPy to store the indices and values of non-zero coordinates rather than the entire sparse vector, reducing GPU memory from $d$ to $(d+\log d)*r$.


Pruning can be seen as radically setting the direction of the gradient for frozen parameters to be sampled from $N(0,0)$. We further propose $\textbf{Importance Matrix}$ to guide the training of the remaining parameters. Parameters with high scores are sampled from $N(0,x)(x> 1)$ with larger weight, tuned with more effort. The experiment result is shown in Table \ref{table robert}.
We propose rank-based important matrix whose standard deviation follows the uniform distribution from A to B based on the rank of parameters remained.
\begin{algorithm}[htb]
    \SetAlgoLined
    \caption{Importance Matrix$(Score,r,type)$}
    \label{alg:4}
    \Input{Previous pruning matrix $M_0$, pruning rate $r$, Matrix type $type$, list $Score$, upper bound $A$,lower bound $B$}
    Initial importance matrix $M={0}_{\left[d\times d\right]}$\;
    \If{$type==Incremental$}
    {
    $Score[i]=-\infty\ \ \mathrm{if}\ \ M_0[i][i]\neq 0$ 
    }
    $Score=Score$.sort()\;
    $Score=Score[r\cdot \mathrm{len}(Score):]$\;
    \For{$i$-th dimension $\theta_i$ of parameter $\boldsymbol{\theta}$}
    {
    \If{$\theta_i.score\in Score$}{
    \If{$type==pruning-only$}
    {
    $M[i][i]=1$
    }
    \If{$type==rank-based$}
    {
    $rank=Score.\mathrm{index}(\theta_i.score)$\;
    $M[i][i]=A-\frac{(A-B)\cdot rank}{r\cdot\mathrm{len}(Score)}$
    }
    }
    }
    \textbf{return} importance matrix $M$
\end{algorithm}

\subsubsection{Dynamic Pruning Mask}
As iterations progress, the direction in which parameters require the most updates keeps changing, making the dynamic pruning mask essential. To tackle the above problem, we propose a dynamic pruning approach with an increasing pruning rate across stages. At the beginning of each stage, the pruning mask is updated based on the score calculated from the current parameters and will remain unchanged throughout the current stage. The pruning rate will increase as the stage progresses, meaning that more parameters will be fine-tuned, further exploiting the potential of the model.
In the initial stage, DP zeroth-order fine-tuning with a smaller pruning rate accelerates the model to converge, not only because it reduces the scale of DP noise introduced but also because it diminishes the gradient approximation error as the dimensionality decreases. In the following stages, the continuously increasing pruning rate overcomes the limitation of the static pruning strategy, which can only fine-tune a fixed subset of parameters, failing to explore the full potential of the large language model. In Algorithm \ref{alg:5}, the dynamic pruning strategy differs from the static pruning strategy in lines 1 and 3 where the vector distribution is updated during stages.

\begin{algorithm}[htb]
    \SetAlgoLined
    \caption{DP-ZOPO (\textbf{D}ifferentially \textbf{P}rivate Stagewise \textbf{P}runing \textbf{O}ptimizer}
    \label{alg:5}
    \Input{Initial point $\boldsymbol{\theta}^{0}\in\mathcal{R}^d$, initial stepsize $\eta_0\ge 0$, initial ZO scale parameter $\beta_0$, dynamic pruning rate $r_s$, regularization parameter $\lambda$, initial iteration number $T_0$,  number of epoch $S$}
    $\mathbb{V} \gets \mathrm{Algorithm 3}(\boldsymbol{\theta}^{0},r,type)$ if $strategy==static$\;
    \For{$s=1$ to $S$}
    {
     $\small\mathbb{V}_s \gets \mathrm{Algorithm 3}(\boldsymbol{\theta}^{s-1},r_s,type)$ if $strategy==dynamic$\;
     $\beta_s=k\cdot\beta_{s-1},\ T_{s}=2\cdot T_{s-1},\ \eta_s=\eta_{s-1}/2$\;
    $\phi_s(\boldsymbol{\theta})=f_{\beta_s}(\boldsymbol{\theta})+\frac{1}{2\lambda}\left\|\boldsymbol{\theta}-\boldsymbol{\theta}^{s-1}\right\|_2^2$\;
    $\boldsymbol{\theta}^{s} = \textbf{DP-ZOO}(\phi_s(\boldsymbol{\theta}),\boldsymbol{\theta}^{s-1},\beta_s,T_s,\eta_s,\mathbb{V}_s,\mathcal{D})$
    }
    \textbf{return} final model parameter $\boldsymbol{\theta}^{S}$
\end{algorithm}

Due to the introduction of noise in DP fine-tuning, it becomes challenging for the pruning strategy to find the optimal pruning mask. Furthermore, it is deeply concerning that an unreasonable pruning mask can lead to the catastrophic collapse of the entire model. To overcome the above problem, we further propose an incremental dynamic pruning strategy. Incremental dynamic pruning strategy encourages conservative updates of the pruning mask. The parameters requiring fine-tuning in the current stage consist of two parts of equal quantity. One part comprises the parameters that needed updating from the previous stage, while the other part consists of additional parameters calculated based on the scores for the current stage. However, the incremental dynamic pruning strategy will fine-tune fewer parameters compared to the standard dynamic pruning strategy since the parameters that were fine-tuned in the previous stage are not necessarily selected for fine-tuning in the current stage based on the pruning strategy. 

The incremental pruning strategy ensures that the parameters requiring fine-tuning from the previous stage are included in the current stage by setting their scores to negative infinity which is detailed in lines 2-3 of Algorithm \ref{alg:4}.

We present the procedure code of stagewise DP zeroth-order fine tuning with pruning (DP-ZOPO) in Algorithm \ref{alg:5}. If $strategy==static$, the pruning mask is calculated at initial and fixed for the whole fine-tuning process. If $strategy==dynamic$, the pruning mask is dynamically updated for each stage based on the current pruning rate (increasing pruning rate is more optimal). If it is an incremental pruning strategy, the parameters to be updated in the current stage will include parameters that were updated in the previous stage.

DP-ZOPO can be divided into two phases. First, we employ data-free pruning to find the important dimensions of parameters $\boldsymbol{\theta}$ (the parameters to be fine-tuned). Next, we use a ZO-based fine-tuning method to optimize the model on distribution $\mathbb{V}$. We denote $\nabla_{\mathbb{V}}f_{\beta}(\boldsymbol{\theta})$ as the zeroth-order gradient on the direction sampled from the distribution $\mathbb{V}$.

\begin{theorem}[Privacy analysis of DP-ZOPO]
    Since pruning in every stage of DP-ZOPO does not require private data, the calculated vector distribution $\mathbb{V}$ will not leak any information about private data. Thus, DP-ZOPO shares the same privacy guarantee with Algorithm DP-ZOSO.
\end{theorem}

\begin{theorem}\label{theo total}
    In DP-ZOPO, suppose assumption \ref{assum 1} holds and loss function $f(\boldsymbol{\theta},x)$ is $\rho$-weakly-convex of $\boldsymbol{\theta}$. The dimensionality of parameters to be updated in stage $s$ is $d\cdot r_s$. Then by setting $\eta_s$ = $\alpha_{s}$$\cdot\min\{\frac{Pm}{7\gamma^2}$, $\frac{Pme}{56dr_SC^2}$, $\frac{\epsilon^2 n^2}{7dr_Sc_2^2C^2PT\log(P/\delta)}$, $\frac{Pm}{448dr_S\beta_S^2L^4}\}$ and $\lambda=\frac{7}{2\mu}$, $\eta_sT_s=\frac{7}{2\mu}$, after $S=\lceil\log(\alpha_0/\alpha)\rceil$ stages, we have
    \begin{equation}
        \phi_{s}(\boldsymbol{\theta}'^{S})-\phi_{s}(\boldsymbol{\theta}^*)\leq \alpha.
    \end{equation}
    The total ZO oracle complexity of DP-ZOPO is 
    \begin{equation}
        \small\mathcal{O}\left(\left(\gamma^2+dr_SC^2+dr_S\beta_S^2L^4+\frac{dr_SC^2P^2m\log(P/\delta)}{\epsilon^2 n^2}\right)\cdot\frac{1}{\mu\alpha}\right).
    \end{equation}
\end{theorem}
\begin{remark}
    Compared to the total complexity of DP-ZOSO in Theorem \ref{theo stage}, the total complexity of DP-ZOPO is reduced by a factor of $r_S$ theoretically. In the following Section \ref{sec_exp}, DP-ZOPO  is also proven to obtain better convergence results than DP-ZOSO empirically.
\end{remark}

\section{Experiment}\label{sec_exp}

We conduct comprehensive experiments on both medium-sized masked LM (RoBERTa-large, 350M \cite{liu2019roberta}) and large autoregressive LM (OPT-2.7b \cite{zhang2022opt}) in few-shot settings. We present the experimental results of dynamic ZO scale parameter in Section \ref{sec dynamic zo}, the effectiveness of pruning with zeroth-order and first-order method in Section \ref{sec_zo-fo pruningg}. In Section \ref{sec_exp_pruning_static_mask} and Section \ref{sec_dynamic_mask}, we compare the performance of different pruning strategies. In Section \ref{sec generation task}, we report the experiment results on more diverse tasks.

\subsection{Experimental Setup}
We report the metric on downstream tasks that run with random seed=42. For RoBERTa-large, we conduct experiments under the hyperparameters detailed in Appendix \ref{sec exp} Table \ref{roberta-hyper}. We use the hyperparameters in Appendix \ref{sec exp} Table \ref{opt1.3-hyper} and Table \ref{opt2.7-hyper} for our experiments on OPT-1.3b and OPT-2.7b. We consider diverse tasks like classification, question-answering, translation and summarization.

\textbf{Classification task:} We consider datasets: SST-2 \cite{socher2013recursive}, SST-5 \cite{socher2013recursive}, SNLI \cite{bowman2015large}, MNLI \cite{williams2017broad} and TREC \cite{voorhees2000building} on RoBERTa-large. We randomly sample 1000 examples for testing and have 512 examples per class for both training and validation. For OPT, we consider the SuperGLUE dataset collection \cite{wang2019superglue} including: CB \cite{de2019commitmentbank}, BoolQ \cite{clark2019boolq} and MultiRC \cite{khashabi2018looking}. We also include SST-2 \cite{socher2013recursive} in our experiments. We randomly sample 1024 examples for training, 500 examples for validation, and 1000 examples for testing.

\textbf{Question-answering task:} We consider SQuAD \cite{rajpurkar2016squad}, a popular benchmark dataset particularly for machine reading comprehension and question answering systems which consists of questions posed by crowdworkers on a set of Wikipedia articles. We randomly sample 1024 examples for training, 500 examples for validation, and 1000 examples for testing.


\textbf{Translation task:} We focus on the translation direction from French to English and fine-tune with the commonly used Europarl \cite{koehn2005europarl} for sentence-level translation. We fine-tune LLMs using examples that include specifically formatted prompts (\colorbox{gray!10}{\textcolor{red}{French:} \textcolor{darkgreen}{[fr sent]} \textcolor{red}{English:}}) and their corresponding responses (\colorbox{gray!10}{\textcolor{darkgreen}{[en sent]}}). Due to the complexity of text-generation tasks, we use 10000 examples for training and 1000 examples for testing. We use BLEU and COMET \cite{rei2020comet} as evaluation metrics to access the performance of fine-tuned models.

\textbf{Summarization task:} We consider CNN/Daily Mail corpus \cite{nallapati2016abstractive} and XSum \cite{narayan2018don}. The CNN/DM dataset primarily focuses on extractive summarization with multi-sentence summaries, while the XSum dataset emphasizes highly abstractive single-sentence summaries. We use the F1 scores of ROUGE-1/2/L/SUM \cite{lin2004rouge} as our evaluation metrics. We randomly sample 1024 examples for training, 500 examples for validation, and 1000 examples for testing.

\textbf{Models:} We conduct experiments on both masked language model (RoBERTa-large, 350M \cite{liu2019roberta}) and autoregressive language model (OPT-1.3b, OPT-2.7b \cite{zhang2022opt}) in few-shot settings. We include a range of ablation studies, including varying the privacy budget, the pruning strategy and the pruning rate. We also explore both full-parameter tuning and parameter-efficient fine-tuning like prefix-tuning. 

\textbf{Privacy Budgets:} We consider various privacy levels with $\epsilon=[2,4,8]$ and dynamic $\delta=1/n$ where $n$ is the number of private data for $(\epsilon,\delta)$-DP. We also include $\epsilon=\infty$ baseline that is trained without DP noise. In our experiments, we set the clipping threshold to $C=30$. We include the zero-shot does not incur any privacy loss because we evaluate the pre-trained model directly without fine-tuning on private data.

\subsection{Main Results}
We conduct experiments with RoBERTa-large on sentiment classification, natural language inference, and topic classification. We follow \cite{malladi2023fine} \cite{gao2020making} to study the few-shot and many-shot settings, sampling $k$ examples per class for $k=512$. DPZero \cite{zhang2024dpzero} is tested with moment accountant just for fairness. We summarize the results from Table \ref{table robert}.

\begin{table}[tb]
\vspace{1.0em}
\caption{Experiments on RoBERTAa-large. Our method outperforms zero-shot, DPZero and approaches FT with much less memory.}
\centering
\label{table robert}
\setlength{\tabcolsep}{1.9mm}{
\begin{tabular}{c c c c c c c}
\toprule
\multicolumn{2}{l}{Task} & $\textbf{SST-2}$ & $\textbf{SST-5}$ & $\textbf{SNLI}$ & $\textbf{MNLI}$ & $\textbf{TREC}$\\
\multicolumn{2}{l}{Type} & \multicolumn{2}{c}{---sentiment---} & \multicolumn{2}{c}{NL inference} & topic \\
\hline
\specialrule{0em}{1pt}{1pt}
\multicolumn{2}{l}{Zero-shot} & 79.0 & 35.5 & 50.2 & 51.4 & 32.0 \\
\hline
\hline
\specialrule{0em}{1pt}{1pt}
\multicolumn{7}{c}{Small Privacy budget: $\epsilon=2$} \\
\hline
\specialrule{0em}{1pt}{1pt}
\multicolumn{2}{l}{FT} & 93.0 & 52.1 & 86.4 & 82.2 & 96.4 \\
\multicolumn{2}{l}{FT-prefix} & 92.2 & 45.2 & 76.4 & 64.0 & 84.4 \\
\multicolumn{2}{l}{DPZero} & 93.1 & 46.1 & 78.0 & 70.1 & 87.4 \\
\multicolumn{2}{l}{DP-ZOSO} & 92.9 & 45.8 & 76.8 & 66.3 & 83.4 \\
\multicolumn{2}{l}{DP-ZOPO} & 92.8 & 50.5 & 80.1 & 72.0 & 91.4 \\
\hline
\hline
\specialrule{0em}{1pt}{1pt}
\multicolumn{7}{c}{Medium Privacy budget: $\epsilon=4$} \\
\hline
\specialrule{0em}{1pt}{1pt}
\multicolumn{2}{l}{FT} & 93.8 & 52.1 & 87.0 & 82.3 & 96.6 \\
\multicolumn{2}{l}{FT-prefix} & 92.1 & 45.2 & 76.8 & 64.0 & 84.4 \\
\multicolumn{2}{l}{DPZero} & 93.1 & 46.6 & 80.4 & 69.4 & 87.6 \\
\multicolumn{2}{l}{DP-ZOSO} & 92.6 & 42.0 & 78.9 & 67.3 & 83.8 \\
\multicolumn{2}{l}{DP-ZOPO} & 93.6 & 51.1 & 80.7 & 74.9 & 92.0 \\
\hline
\hline
\specialrule{0em}{1pt}{1pt}
\multicolumn{7}{c}{Large Privacy budget: $\epsilon=8$} \\
\hline
\specialrule{0em}{1pt}{1pt}
\multicolumn{2}{l}{FT} & 93.8 & 52.6 & 87.2 & 83.0 & 96.6 \\
\multicolumn{2}{l}{FT-prefix} & 92.2 & 44.8 & 76.9 & 63.9 & 84.6 \\
\multicolumn{2}{l}{DPZero} & 93.4 & 47.0 & 80.9 & 69.2 & 87.2 \\
\multicolumn{2}{l}{DP-ZOSO} & 93.2 & 41.2 & 76.5 & 69.2 & 84.2 \\
\multicolumn{2}{l}{DP-ZOPO} & 93.6 & 49.8 & 81.6 & 74.8 & 91.8 \\

\bottomrule
\end{tabular}}
\end{table}

\textbf{DP-ZOPO works significantly better than zero-shot and other memory-equivalent methods.} On all five tasks, our method optimizes the pre-trained model and consistently performs better than zero-shot and prefix-tuning. We also show for several tasks that DP-ZOPO can outperform DPZero up to 5\% even with same privacy analysis (improve 4.4\% on TREC, 5.5\% on MNLI with $\epsilon=4$). DP-ZOPO outperforms 1.6\% on SST-2 with $\epsilon=4$ compared with the results shown in \cite{tang2024private}.

\textbf{Pruning improves model performance in private settings.} We compare DP-ZOPO (with pruning) and DP-ZOSO (without pruning). We show the performance of DP-ZOPO (static pruning mask) among six different pruning rates in Figure \ref{fig2}. Pruning improves performance greatly in all five tasks in all private settings ($\epsilon=2,4,8$). As a detailed result, the best performance of DP-ZOPO is $74.9\%$ on MNLI, outperforming DP-ZOSO by $7.6\%$ with $\epsilon=4$. In different private settings, the optimal pruning rate varies which will be discussed in Section \ref{sec_exp_pruning_static_mask}. With the promising results from RoBERTa-large, we extend our method to the OPT-2.7b \cite{zhang2022opt}. We select SuperGLUE tasks \cite{wang2019superglue} (including classification and multiple-choice). We randomly sample 1024 examples for training and 1000 test examples for each dataset.

\begin{table}[tb]
\vspace{1.0em}
\caption{Experiments on OPT-2.7b with privacy budget $\epsilon=4,\delta=1/1024$. ICL for in-context learning; FT-prefix for prefix-tuning. DP-ZOSO and DP-ZOPO performs better than zero-shot, ICL, and DPZero with almost the same memory consumption. }
\centering
\label{table opt}
\setlength{\tabcolsep}{1.5mm}{
\begin{tabular}{c c c c c c}
\toprule
\multicolumn{2}{l}{Task} & $\textbf{SST-2}$ & $\textbf{CB}$ & $\textbf{BoolQ}$ & $\textbf{MultiRC}$ \\
\hline
\hline
\specialrule{0em}{1pt}{1pt}
\multicolumn{2}{l}{Zero-shot} & 56.3 & 50.0 & 48.0 & 44.3 \\
\multicolumn{2}{l}{ICL} & 77.6 & 62.5 & 57.9 & 47.8 \\
\hline
\hline
\specialrule{0em}{1pt}{1pt}
\multicolumn{2}{l}{DPZero} & 91.5 & 69.6 & 63.6 & 61.9 \\
\multicolumn{2}{l}{DP-ZOSO} & 90.0 & 62.0 & 63.7 & 50.0 \\
\multicolumn{2}{l}{DP-ZOPO (static)} & \textbf{93.2} & 69.7 & 63.7 & 61.9 \\
\multicolumn{2}{l}{DP-ZOPO (dynamic)} & 92.5 & \textbf{78.8} & \textbf{67.5} & \textbf{65.0} \\
\bottomrule
\end{tabular}}
\end{table}

\textbf{DP-ZOPO works significantly better than zero-shot and ICL.} 
On both classification and multiple-choice tasks, DP-ZOPO exhibits strong performance. On a 2.7b-parameter scale, DP-ZOPO outperforms zero-shot and ICL across all tasks with equivalent memory consumption (details in Table \ref{table opt}). DP-ZOPO outperforms DPZero \cite{zhang2024dpzero} in all five tasks on autoregressive language OPT-2.7b, we improve the accuracy by $3.9\%$ on BoolQ and $9.2\%$ on CB with $\epsilon=4$. We believe that the quality of data-free pruning can greatly influence optimization and we leave how to enhance the performance of pruning as future work.

\textbf{Dynamic pruning outperforms Static pruning.} In dynamic pruning, we select which parameters to update at each stage and update those selected using DP-ZOO. We achieve an improvement of $9.1\%$ on CB and 3.8\% on BoolQ with OPT-2.7b. Same conclusions are also drawn on the RoBERTa-large in Table \ref{table_roberta_cm&dm}. Moreover, incremental pruning strategy is superior to other pruning strategies in most scenarios.

\subsection{Memory Usage}
In this section, we profile the memory usage of zero-shot, ICL, FT, FT-prefix, DP-ZOSO and DP-ZOPO. We test OPT-1.3b and OPT-2.7b with Nvidia A100 GPUs (40GB memory) on SST-2 and report the GPU memory consumption in Figure \ref{fig:memory}. DP-ZOSO costs the same GPU memory consumption as Zero-shot and ICL which only use inference for updates. DP-ZOPO uses pruning mask (the same dimensionality as the model) to freeze model parameters during training which costs quite large GPU memory. However, the pruning rate is usually set quite small which means the pruning mask is a sparse vector. We store non-zero ids of the mask instead of the whole vector to reduce memory consumption overhead. Pruning costs a slight amount of (8\%) GPU memory while improving the model utility greatly.

\begin{figure}[tbp]
    \centering
    \includegraphics[width=0.95\linewidth]{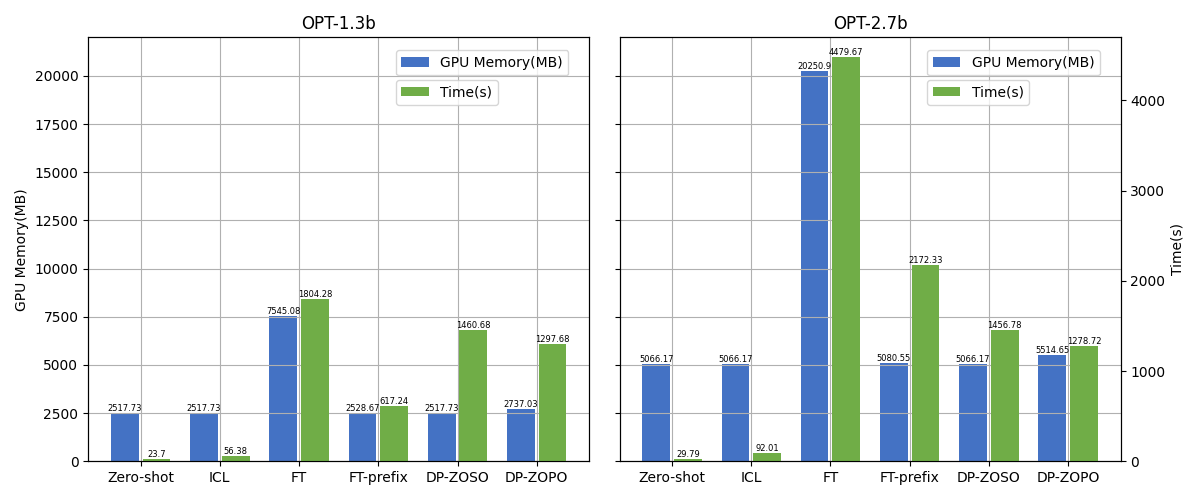}
    \caption{The GPU memory consumption and running time with OPT-1.3b and OPT-2.7b on SST-2. DP-ZOSO and DP-ZOPO cost less GPU memory consumption and running time.}
    \label{fig:memory}
\end{figure}

As shown in Figure \ref{fig:memory}, DP-ZOPO exhibits almost the same ($1.08\times$) memory consumption which offers memory savings of up to 4 times compared to FT. Moreover, if we pre-prune the network and update the remaining parameters using DP-ZOO, the pruning mask can be calculated offline and stored locally, without incurring additional GPU memory overhead. DP-ZOPO also runs faster than first-order methods, iterating 6000 rounds with $58.8\%$ time of FT-prefix and $28.5\%$ of FT.



\subsection{Dynamic ZO scale parameter}\label{sec dynamic zo}
We fine-tune RoBERTa-large on SST-2 with both fixed and dynamic ZO scale parameter to demonstrate the superiority of dynamic ZO scale parameter. We present the final training loss of the fine-tuned model with different ZO scale parameter in Table \ref{table dynamic scale}.
\begin{table}[hb]
\vspace{1.0em}
\caption{Training loss of RoBERTa-large on SST-2 with both fixed and dynamic ZO scale parameter. Dynamic ZO scale parameter is reduced during stages. We denote M1 as the scale parameter decreasing from $1e-5 \sim 1e-4$, M2 as $1e-6 \sim 1e-5$ and M3 as $1e-6 \sim 1e-4$.}
\centering
\label{table dynamic scale}
\setlength{\tabcolsep}{1.2mm}{
\begin{tabular}{c c c|c c c}
\toprule 
\multicolumn{3}{c}{Fix ZO scale parameter} & \multicolumn{3}{c}{Dynamic ZO scale parameter} \\
\hline
\hline
\specialrule{0em}{1pt}{1pt}
1e-4 & 1e-5 & 1e-6 & M1 & M2 & M3 \\
\hline
\specialrule{0em}{1pt}{1pt}
0.31532 & 0.31530 & 0.31639 & 0.31812 & 0.31434 & 0.31368 \\
\bottomrule
\end{tabular}}
\end{table}

All experiments are done under $\epsilon=4$ with 6k steps. M3 represents the algorithm with ZO scale parameter increasing from $1e-6$ to $1e-4$ as the stage advances. It is shown in Table \ref{table dynamic scale} that dynamic ZO scale parameter method M3 ($1e-6 \sim 1e-4$) exhibits the lowest loss on the training set compared with both fixed ZO scale parameter $1e-4$ and $1e-6$. Dynamic M2 and M3 both perform better than all fixed ZO scale parameter, further elaborating on the superiority of the dynamic ZO scale parameter approach. For the following experiments of RoBERTa-large, we set the dynamic ZO scale parameter from $1e-6$ to $1e-4$.

\subsection{Zeroth-Order Pruning VS First-Order Pruning}\label{sec_zo-fo pruningg}
In this section, we compare the effectiveness of pruning in both zeroth-order and first-order methods. In FO method, the gradient is calculated by backpropagation, and only $r$ of the parameters are updated. First-order method in conjunction with pruning reduces the scale of DP noise but also diminishes the model's utility since only a subset of parameters is fine-tuned. 

\begin{figure}[htbp]
    \centering
    \includegraphics[width=0.95\linewidth]{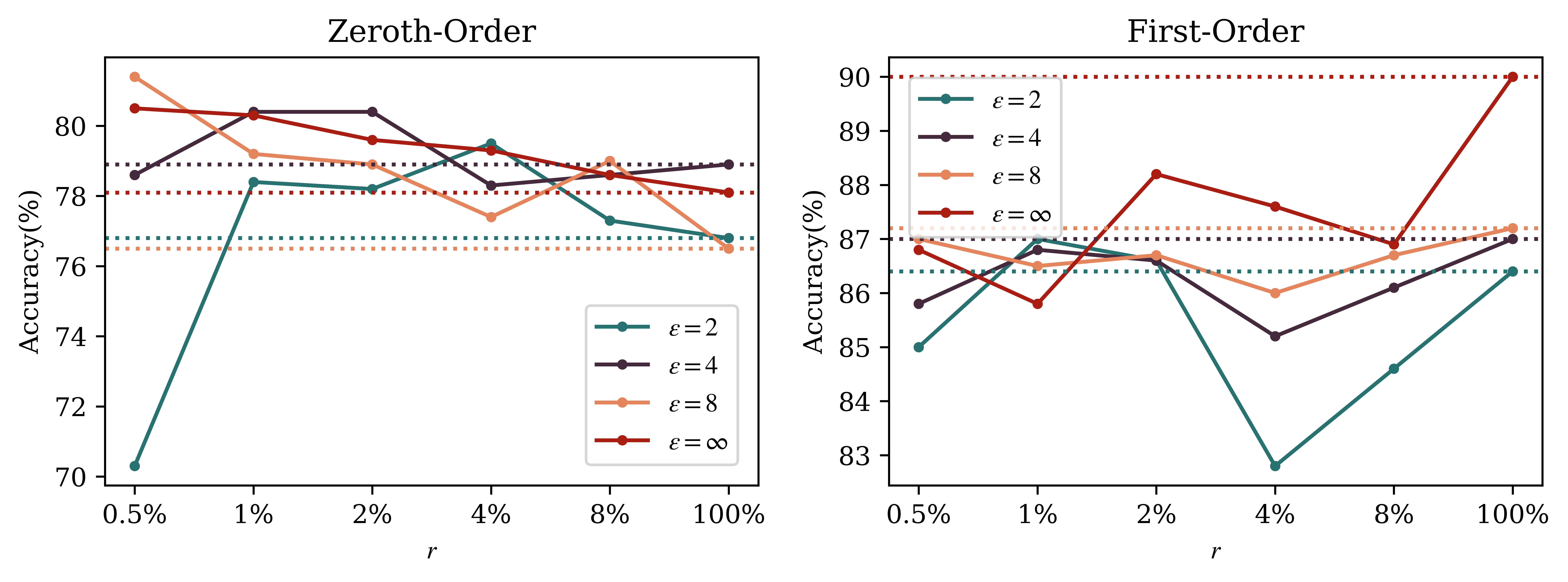}
    \caption{Results of fine-tuned RoBERTa-large on SNLI with zeroth-order and first-order method. In zeroth-order method, fine-tuning with pruning helps with optimization under all private settings.}\label{fig_pruning_zovsfo}
\end{figure} 

In Figure \ref{fig_pruning_zovsfo}, we find that in first-order method, pruning helps with fine-tuning when the privacy budget is small and does not significantly harm fine-tuning, only reducing the accuracy by $0.2\%$ with larger privacy budget. However, in zeroth-order method, pruning not only scales down the DP noise but also diminishes the gradient approximation error as the dimensionality decreases. Under all private settings ($\epsilon=2,4,8,\infty$), pruning helps with zeroth-order fine-tuning and improves 2.7\% when $\epsilon=2$ (more detailed numbers in Table \ref{table_roberta_firstvszero_pruning}). The success of the zeroth-order method with pruning motivates us to investigate the impact of different pruning strategies on DP zeroth-order fine-tuning. We will present more experiment results of different pruning strategies in the following section.

\subsection{Data-free Pruning with static mask}\label{sec_exp_pruning_static_mask}
We first conduct experiments with different learning rates among $\left\{1e-5,2e-5,\dots,1e-4,2e-4\right\}$, aiming to find appropriate learning rates for different pruning rates. We conduct our following experiments under appropriate learning rates. We further study the trends in optimal pruning rate selection under different privacy settings $\epsilon=2,4,8,\infty$.

\textbf{Pruning works well in private ZO fine-tuning.} We show our results in Figure \ref{fig2} that under all private settings, zeroth-order fine-tuning with pruning achieves better performance (above 2\%, more detailed numbers in Table \ref{table_roberta_firstvszero_pruning}). Especially when the privacy budget is large like $\epsilon=8,\infty$, zeroth-order fine-tuning with pruning outperforms without pruning by $4\sim5\%$.

\begin{figure}[t]
    \centering
    \includegraphics[width=0.8\linewidth]{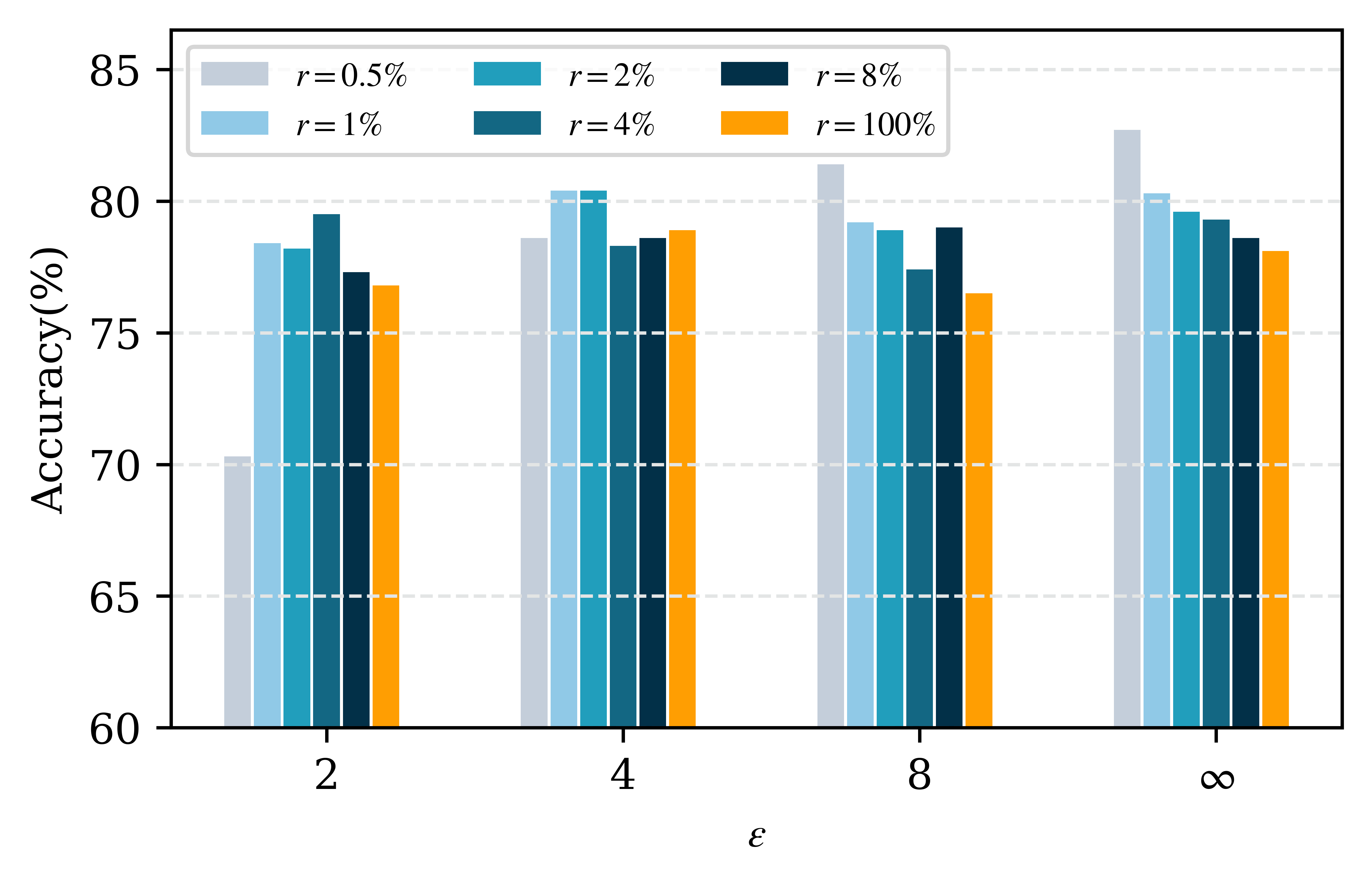}
    \caption{Fine-tuning RoBERTa-large model on dataset SNLI with different pruning rates. Pruning rate $r$ denotes the ratio of parameters to be tuned. Optimal $r$ scales up with the decrease of privacy budget. Detailed numbers in Table \ref{table_roberta_firstvszero_pruning}.}\label{fig2}
\end{figure} 

\textbf{Optimal pruning rate changes with privacy budget.} In our experiments, the pruning rate is seen as a hyperparameter. However, the choice of pruning rate affects the model utility on downstream tasks (see Figure \ref{fig2}). When the privacy budget is large (like $\epsilon=8$), useful parameters can be tuned well, small pruning rate can be chosen to lower the scale of DP noise discussed in Section \ref{sec pruning}. When privacy budget is small (like $\epsilon=2$), $0.5\%$ parameters can not be tuned well for large DP noise, thus bigger pruning rate $r=4\%$ should be chosen for better convergence.


We further conduct experiments on pruning-only method and rank-based important matrix with different intervals (upper bound and lower bound in Algorithm \ref{alg:4}). It is shown in Table \ref{our vs method} that rank-based important matrix outperforms pruning-only method under medium and large privacy budgets while behaves poorly under small privacy budgets.

\begin{table}[tb]
\vspace{1.0em}
\caption{Rank-based important matrix with different intervals on RoBERTa-large. [1.0-1.0] performs best with small privacy budget $\epsilon=2$.}
\centering
\label{our vs method}
\begin{tabular}{c c c c c c}
\toprule
\multicolumn{2}{l}{\textbf{Interval}} & [0.7-1.3] &[0.8-1.2] & [0.9-1.1] & [1.0-1.0] \\
\hline
\hline
\specialrule{0em}{1pt}{1pt}
\multicolumn{6}{c}{Small Privacy budget: $\epsilon=2$} \\
\hline
\specialrule{0em}{1pt}{1pt}
\multicolumn{2}{c}{SST-2} & 91.7 & 92.1 & 92.4 & \textbf{92.7} \\
\multicolumn{2}{c}{SST-5} & 47.3 & 48.3 & 48.1 & \textbf{48.8} \\
\multicolumn{2}{c}{SNLI} & 78.6 & 78.8 & \textbf{78.9} & \textbf{78.9} \\
\multicolumn{2}{c}{MNLI} & 68.7 & 70.2 & 70.1 & \textbf{71.4} \\
\multicolumn{2}{c}{TREC} & 88.4 & 89.0 & 88.2 & \textbf{89.6} \\
\hline
\hline
\specialrule{0em}{1pt}{1pt}
\multicolumn{6}{c}{Medium Privacy budget: $\epsilon=4$} \\
\hline
\specialrule{0em}{1pt}{1pt}
\multicolumn{2}{c}{SST-2} & 92.2 & \textbf{93.3} & 93.2 & 93.0 \\
\multicolumn{2}{c}{SST-5} & 49.2 & 49.2 & \textbf{51.1} & 50.5 \\
\multicolumn{2}{c}{SNLI} & \textbf{80.7} & 80.4 & 80.1 & 80.4 \\
\multicolumn{2}{c}{MNLI} & 69.2 & 71.1 & 70.6 & \textbf{71.5} \\
\multicolumn{2}{c}{TREC} & 90.2 & \textbf{90.4} & 88.8 & 88.2 \\
\hline
\hline
\specialrule{0em}{1pt}{1pt}
\multicolumn{6}{c}{Large Privacy budget: $\epsilon=8$} \\
\hline
\specialrule{0em}{1pt}{1pt}
\multicolumn{2}{c}{SST-2} & 90.8 & \textbf{93.5} & \textbf{93.5} & 93.2 \\
\multicolumn{2}{c}{SST-5} & 49.6 & 48.5 & 49.1 & \textbf{49.8} \\
\multicolumn{2}{c}{SNLI} & 78.7 & 80.1 & 78.0 & \textbf{81.4} \\
\multicolumn{2}{c}{MNLI} & 70.8 & 70.2 & 69.4 & \textbf{73.1} \\
\multicolumn{2}{c}{TREC} & 89.2 & 89.4 & \textbf{91.4} & 90.8 \\
\bottomrule
\end{tabular}
\vspace{-1em}
\end{table}

\begin{figure*}[t]
    \centering
    \includegraphics[width=0.75\linewidth]{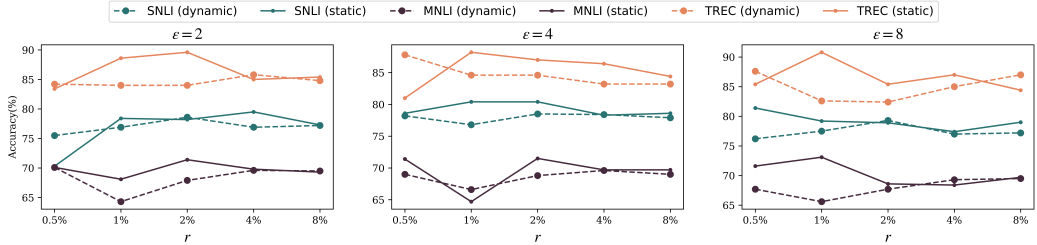}
    \caption{RoBERTa-large on five datasets. Dynamic pruning with a constant pruning rate fails to improve optimization.}\label{fig_dynamic_pruning(constant)&static}
\end{figure*} 

\begin{figure*}[htbp]
    \centering
    \includegraphics[width=\linewidth]{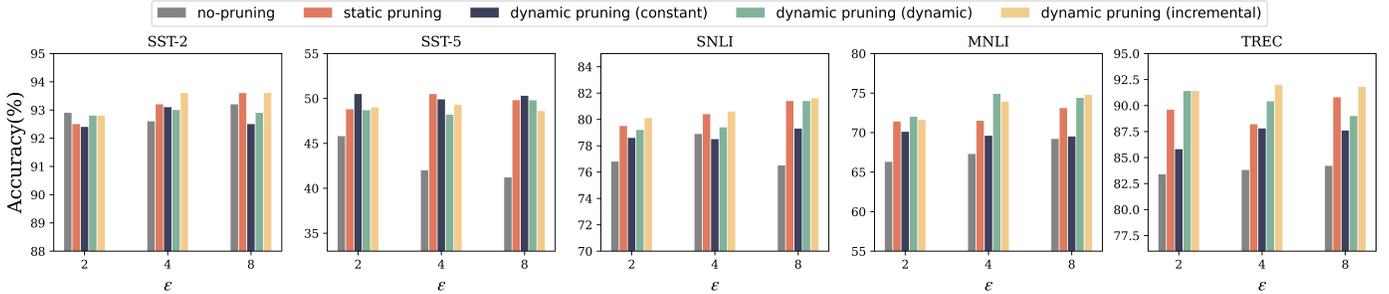}
    \caption{RoBERTa-large fine-tuned with different pruning strategies. Incremental pruning outperforms other pruning strategies.}\label{fig_result_of_pruning_strategy}
\end{figure*} 
We point out that the optimal interval varies with privacy budget $\epsilon$. Taking the interval [0.8-1.2] as an example, the step of zeroth-order gradient on the most important parameter will be sampled from $N(0,1.2)$, thus the faster to converge but at the same time the noise introduced is also larger than other parameters in expectation which makes the parameter hard to find its saddle point. Things get extremely worse with small privacy budget. When $\epsilon=2$, parameters can not be tuned well in the first place, fine-tuning with important matrix makes it even worse. Thus, the optimal interval of important matrix is [1.0-1.0] with small privacy budget $\epsilon=2$. However, optimal interval under different private settings remains an open question to discuss.

\subsection{Data-free Pruning with dynamic mask}\label{sec_dynamic_mask}
In this section, we present two parts of experiments regarding different dynamic pruning strategies. The first part of the experiments follows the pruning strategy where the pruning mask is dynamically updated for each stage, while the pruning rate remains constant. In the second part of the experiments, the pruning mask is dynamically updated for each stage with an increasing pruning rate. We also conduct experiments on the incremental pruning strategy which is a more conservative pruning approach.

\textbf{Dynamic pruning with constant pruning rate. }We compare dynamic and constant pruning strategies under the same (constant) pruning rate. We find that dynamic updating of the pruning mask does not necessarily lead to an improvement in accuracy with the same (fixed) pruning rate. Figure \ref{fig_dynamic_pruning(constant)&static} shows the accuracy under static and dynamic pruning strategy (constant pruning rate) varying pruning rate on five datasets. With constant pruning rate, the static pruning strategy outperforms dynamic pruning strategy in most cases. In Table \ref{table_roberta_cm&dm}, we present the best results of both pruning strategies across all pruning rates which indicates that dynamic pruning with constant pruning rate fails to improve optimization.

We blame the failure on two main reasons: Firstly, in dynamic pruning, the parameters that were fine-tuned in the previous stage were not fully optimized and may not be selected for further fine-tuning in the next stage. As a result, the model may not be well-optimized. Secondly, the introduction of DP noise leads to the low quality of the update of the pruning mask. In Table \ref{table parameter in next stage}, we show the proportion of the parameters fine-tuned in the current stage that are selected for fine-tuning in the next stage. Additionally, we provide the proportion of all parameters involved in the fine-tuning process compared to the total number of model parameters and the final accuracy of the model.

\begin{table}[tb]
\vspace{1em}
\caption{RoBERTa-large on SNLI with $\epsilon=4$. ``Stage 1'' denotes the proportion of parameters fine-tuned in stage 1 that will be fine-tuned in the next stage (stage 2), ``Total'' denotes the proportion of parameters fine-tuned in the whole fine-tuning process, and ``Acc'' denotes the accuracy of the model. ``1-2-4\%''$\ast$ denotes the incremental pruning strategy with increasing pruning rate 1-2-4\%.}
\centering
\linespread{1.5}
\label{table parameter in next stage}
\setlength{\tabcolsep}{1.3mm}{
\begin{tabular}{c|c|c|c|c|c}
\toprule
\specialrule{0em}{1pt}{1pt}
\multicolumn{1}{l}{Method} & \multicolumn{1}{|c|}{Type} & Stage 1 & Stage 2 & Total & Acc \\
\midrule
\midrule
\multicolumn{1}{l|}{\multirow{5}{*}{\small constant $r$}} & \multicolumn{1}{|c|}{\small 0.5-0.5-0.5\%} & 0.54\%  & 0.86\% & 0.75\% & 78.2\% \\
\multicolumn{1}{l|}{} & \multicolumn{1}{|c|}{\small 1-1-1\%} & 1.74\%  & 2.56\% & 1.47\% & 76.8\% \\
\multicolumn{1}{l|}{} & \multicolumn{1}{|c|}{\small 2-2-2\%} & 3.30\% & 4.79\% & 2.89\% & 78.5\% \\
\multicolumn{1}{l|}{} & \multicolumn{1}{|c|}{\small 4-4-4\%} & 8.98\% & 11.11\% & 5.48\% & 78.4\% \\
\multicolumn{1}{l|}{} & \multicolumn{1}{|c|}{\small 8-8-8\%} & 17.60\% & 19.13\% & 10.15\% & 77.9\% \\
\midrule
\multicolumn{1}{l|}{\multirow{6}{*}{\small dynamic $r$}} & \multicolumn{1}{|c|}{\small 0.5-1-2\%} & 1.01\% & 2.86\% & 1.73\% & 79.2\% \\
\multicolumn{1}{l|}{} & \multicolumn{1}{|c|}{\small 0.5-1-2\%$\ast$} & 100\% & 100\% & 2\% & 78.4\% \\
\multicolumn{1}{l|}{} & \multicolumn{1}{|c|}{\small 1-2-4\%} & 3.19\% & 8.27\% & 3.38\% & 79.4\% \\
\multicolumn{1}{l|}{} & \multicolumn{1}{|c|}{\small 1-2-4\%$\ast$} & 100\% & 100\% & 4\% & 80.1\% \\
\multicolumn{1}{l|}{} & \multicolumn{1}{|c|}{\small 2-4-8\%} & 6.00\% & 15.00\% & 6.55\% & 78.7\% \\
\multicolumn{1}{l|}{} & \multicolumn{1}{|c|}{\small 2-4-8\%$\ast$} & 100\% & 100\% & 8\% & \textbf{80.6\%} \\
\bottomrule
\end{tabular}}
\vspace{-1em}
\end{table}

\begin{table}[htb]
\vspace{1em}
\caption{Experiments of four pruning strategies on RoBERTAa-large. All the results are the best result across all pruning rates. ``static*'' denotes the best result when the pruning mask is calculated before fine-tuning and remains static during the whole fine-tuning process. ``dynamic(constant $r$)'' and ``dynamic(dynamic $r$)'' denote the dynamic pruning strategy with constant and dynamic pruning rate.}
\centering
\label{table robert_all4pruning strategy}
\setlength{\tabcolsep}{1.2mm}{
\begin{tabular}{c c c c c c}
\toprule
\multicolumn{1}{l}{Task} & $\textbf{SST-2}$ & $\textbf{SST-5}$ & $\textbf{SNLI}$ & $\textbf{MNLI}$ & $\textbf{TREC}$\\
\hline
\hline
\specialrule{0em}{1pt}{1pt}
\multicolumn{6}{c}{Small Privacy budget: $\epsilon=2$} \\
\hline
\specialrule{0em}{1pt}{1pt}
\multicolumn{1}{l}{\small static*} & 92.5 & 48.8 & 79.5 & 71.4 & 89.6 \\
\multicolumn{1}{l}{\small dynamic(constant $r$)} & 92.4 & \textbf{50.5} & 78.6 & 70.1 & 85.8 \\
\multicolumn{1}{l}{\small dynamic(dynamic $r$)} & \textbf{92.8} & 48.7 & 79.2 & \textbf{72.0} & \textbf{91.4} \\
\multicolumn{1}{l}{\small dynamic(incremental)} & \textbf{92.8} & 49.0 & \textbf{80.1} & 71.6 & \textbf{91.4} \\
\hline
\hline
\specialrule{0em}{1pt}{1pt}
\multicolumn{6}{c}{Medium Privacy budget: $\epsilon=4$} \\
\hline
\specialrule{0em}{1pt}{1pt}
\multicolumn{1}{l}{\small static*} & 93.2 & \textbf{50.5} & 80.4 & 71.5 & 88.2 \\
\multicolumn{1}{l}{\small dynamic(constant $r$)} & 93.1 & 49.9 & 78.5 & 69.6 & 87.8 \\
\multicolumn{1}{l}{\small dynamic(dynamic $r$)} & 93.0 & 48.2 & 79.4 & \textbf{74.9} & 90.4 \\
\multicolumn{1}{l}{\small dynamic(incremental)} & \textbf{93.6} & 49.3 & \textbf{80.6} & 73.9 & \textbf{92.0} \\
\hline
\hline
\specialrule{0em}{1pt}{1pt}
\multicolumn{6}{c}{Large Privacy budget: $\epsilon=8$} \\
\hline
\specialrule{0em}{1pt}{1pt}
\multicolumn{1}{l}{\small static*} & 93.6 & 49.8 & 81.4 & 73.1 & 90.8 \\
\multicolumn{1}{l}{\small dynamic(constant $r$)} & 92.5 & \textbf{50.3} & 79.3 & 69.5 & 87.6 \\
\multicolumn{1}{l}{\small dynamic(dynamic $r$)} & 92.9 & 49.8 & 81.4 & 74.4 & 89.0 \\
\multicolumn{1}{l}{\small dynamic(incremental)} & \textbf{93.6} & 48.6 & \textbf{81.6} & \textbf{74.8} & \textbf{91.8} \\
\bottomrule
\end{tabular}}
\vspace{-1em}
\end{table}

It is shown in Table \ref{table parameter in next stage} that only a minuscule proportion of parameters are selected for fine-tuning in the next stage. For example, in the first stage of 0.5-0.5-0.5\%, only 0.54\% of the parameters that are updated will continue to be fine-tuned in the next stage, which results in parameters not being well-tuned and they will not be fine-tuned furthermore in the following fine-tuning process. This explains why dynamic pruning with a constant rate may diminish the utility of the model compared to static pruning. However, the above issue is effectively addressed with dynamic (increasing) pruning rate and incremental strategy, which further enhance the model's accuracy. As shown in Table \ref{table parameter in next stage}, dynamic pruning with increasing dynamic 1-2-4\% outperforms dynamic pruning with constant pruning rate of all three 1\%, 2\%, and 4\% (improving accuracy by 2.6\%, 0.9\% and 1\%). Incremental dynamic pruning strategy further improves the model accuracy by 0.7\%.

\textbf{Dynamic pruning with dynamic pruning rate and incremental strategy. }
We compare the best results of all four pruning strategies to explore the potential of different pruning strategies in Table \ref{table robert_all4pruning strategy}. As shown in Figure \ref{fig_result_of_pruning_strategy}, dynamic pruning with dynamic (increasing) pruning rate outperforms dynamic pruning with constant pruning rate. For example, on SNLI with $\epsilon=4$, ``1-2-4\%'' outperforms ``1-1-1\%'' by 2.6\% for the reason that ``1-2-4\%'' fine-tunes more parameters (1.47\% compared with 3.38\%) in the whole fine-tuning stage. Compared with ``4-4-4\%'', ``1-2-4\%'' fine-tunes fewer parameters (5.48\% compared with 3.38\%) and outperforms by 1.0\% because fine-tuning with a smaller pruning rate in the early stages can accelerate the convergence. We show in Figure \ref{fig_training_process} the training-set loss of ``1-2-4\%'', ``1-1-1\%'', ``2-2-2\%'' and ``4-4-4\%''. ``1-1-1\%'' converges faster in the initial stages compared to ``1-2-4\%''. However, since ``1-2-4\%'' fine-tunes more parameters, it achieves better convergence with lower training-set loss. Incremental pruning strategy is a more conservative pruning strategy, it inherits the advantages of the static pruning method, fine-tuning the parameters in the previous stage. For example in Table \ref{table robert incremental}, ``2-4-8\%$\ast$'' outperforms static pruning with $r=2\%$ by 0.2\% since more parameters are fine-tuned. Compared with static pruning with $r=8\%$, ``2-4-8\%$\ast$'' achieves a $2\%$ improvement for fast convergence in the early stage. In summary, incremental pruning strategy performs better in most scenarios when it comes to private zeroth-order fine-tuning. We further present our experimental results on OPT-2.7b in Table \ref{optincre_vs}.

\begin{figure}[t]
    \centering
    \includegraphics[width=0.75\linewidth]{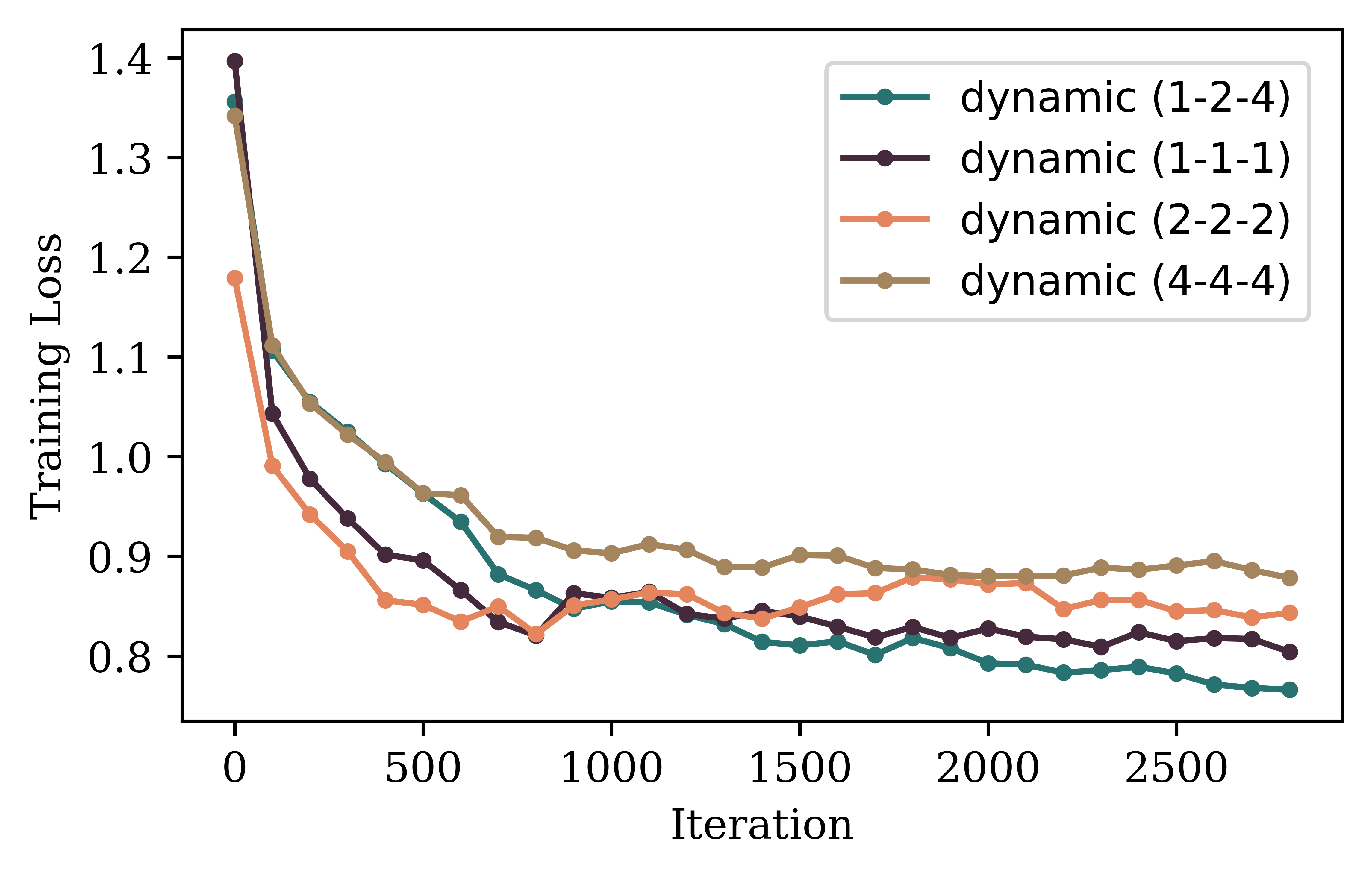}
    \caption{The training-set loss of RoBERTa-large on SNLI with privacy budget $\epsilon=4$. Dynamic pruning with increasing pruning rate ``1-2-4\%'' performs better than both ``1-1-1\%'' and ``4-4-4\%''.}\label{fig_training_process}
\end{figure} 

\begin{figure}[t]
    \vspace{1em}
    \centering
    \includegraphics[width=\linewidth]{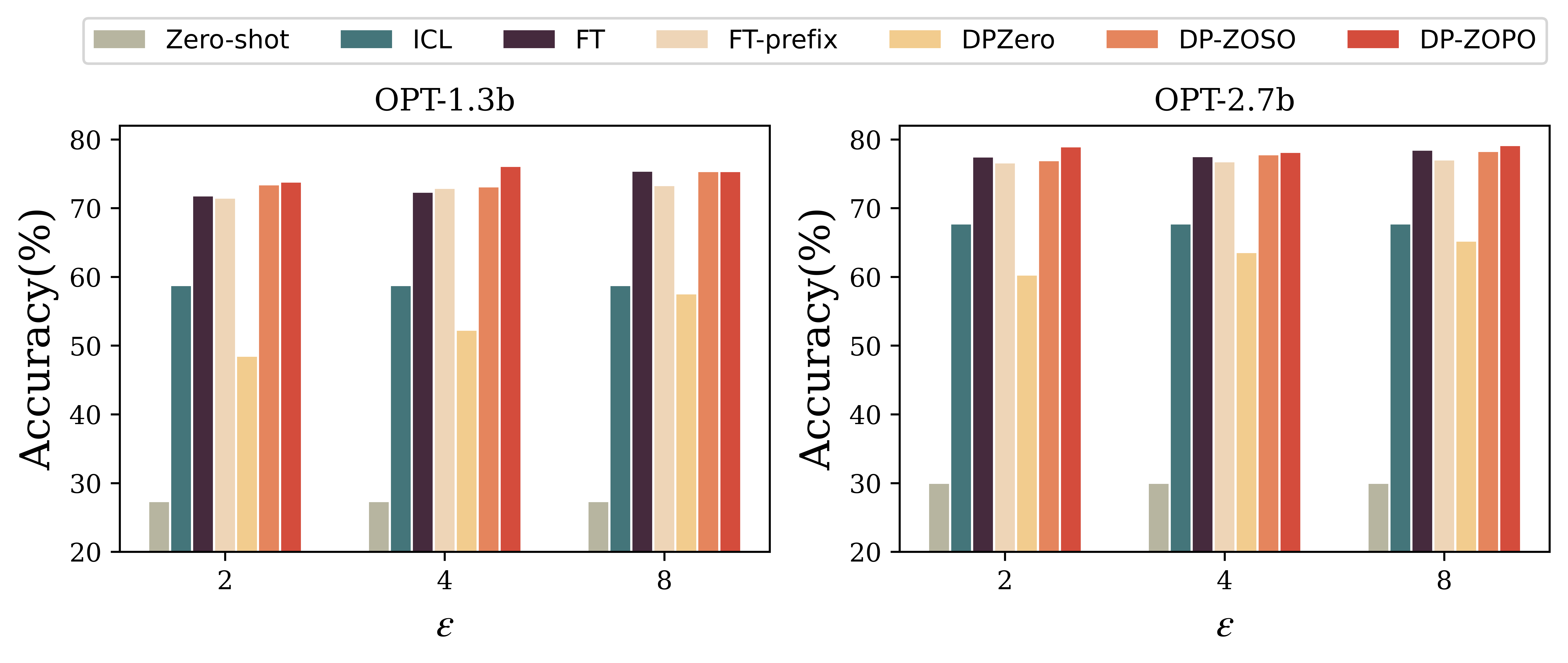}
    \caption{OPT-1.3b and OPT-2.7b fine-tuned on SQuAD. DP-ZOPO performs better than other memory-efficient methods, even FT-prefix and FT under all privacy budgets.}
    \label{fig_QA}
\end{figure}

\subsection{More diverse tasks}\label{sec generation task}
In this section, we present more experimental results pertaining to a broader range of tasks for LLMs, including question-answering, translation, and summarization tasks.

\textbf{Question-answering Task:} We fine-tune OPT-1.3b and OPT-2.7b on SQuAD. As shown in Figure \ref{fig_QA} (detailed in Table \ref{table_squad}), model fine-tuned with DP-ZOPO has better QA accuracy than ICL, DPZero, DP-ZOSO, FT-prefix, even FT in all private settings. Our method far outperforms DPZero \cite{zhang2024dpzero} for its large zeroth-order estimation error, constant learning rate and excessive noise injection.


\begin{table}[t]
    \vspace{1em}
    \centering
    \caption{OPT-1.3b and OPT-2.7b fine-tuned on Europarl.}
    \label{table_translation}
    \begin{tabular}{c c c c c}
    \toprule
        Model & \multicolumn{2}{c}{OPT-1.3b} & \multicolumn{2}{c}{OPT-2.7b} \\
        \hline
        \specialrule{0em}{1pt}{1pt}
        Method & BLEU & COMET & BLEU & COMET \\  
        \hline
        \hline
        \specialrule{0em}{1pt}{1pt}
        Zero-shot & 20.11 & 0.7562 & 25.36 & 0.7990 \\
        ICL & 22.33 & 0.7684 & 27.17 & 0.8105 \\ 
        FT & 21.71 & 0.7695 & 27.11 & 0.8135 \\
        FT-prefix & 20.82 & 0.7609 & 25.21 & 0.8014 \\
        DPZero & 20.34 & 0.7610 & 24.94 & 0.8040 \\
        DP-ZOSO & 21.08 & 0.7567 & 25.44 & 0.8022 \\
        DP-ZOPO & 21.19 & 0.7651 & 26.07 & 0.8044 \\
    \bottomrule
    \end{tabular}
\end{table}

\begin{table}[t]
    \vspace{1em}
    \centering
    \caption{OPT-1.3b fine-tuned on XSum and CNN/DM. DP-ZOPO outperforms ICL and other memory efficient method FT-prefix, DPZero and DP-ZOSO.}
    \label{table_summarization}
    \begin{tabular}{c c c c c c}
    \toprule
         Task & Method & R1 & R2 & RL & R-SUM \\
        \hline
        \hline
        \specialrule{0em}{1pt}{1pt}
        \multirow{6}{*}{XSum} & \multicolumn{1}{l}{Zero-shot} & 18.66 & 3.51 & 13.95 & 13.93 \\
          & \multicolumn{1}{l}{ICL} & 25.15 & 6.92 & 20.07 & 20.07 \\
          & \multicolumn{1}{l}{FT} & 31.10 & 10.55 & 25.06 & 25.03 \\
          & \multicolumn{1}{l}{FT-prefix} & 27.28 & 9.09 & 22.20 & 22.23 \\
          & \multicolumn{1}{l}{DPZero} & 6.40 & 1.20 & 4.66 & 4.66 \\
          & \multicolumn{1}{l}{DP-ZOSO} & 30.13 & 10.34 & 24.43 & 24.45 \\
          & \multicolumn{1}{l}{DP-ZOPO} & 30.49 & 10.44 & 24.64 & 24.63 \\
        \hline
        \specialrule{0em}{1pt}{1pt}
        \multirow{6}{*}{CNN/DM} & \multicolumn{1}{l}{Zero-shot} & 32.54 & 15.07 & 23.11 & 23.17 \\
          & \multicolumn{1}{l}{ICL} & 34.24 & 16.13 & 24.64 & 24.66 \\
          & \multicolumn{1}{l}{FT} & 37.89 & 17.92 & 27.93 & 27.91 \\
          & \multicolumn{1}{l}{FT-prefix} & 31.29 & 14.89 & 22.59 & 22.55 \\
          & \multicolumn{1}{l}{DPZero} & 23.69 & 11.59 & 16.63 & 16.62 \\
          & \multicolumn{1}{l}{DP-ZOSO} & 31.93 & 15.12 & 23.00 & 23.00 \\
          & \multicolumn{1}{l}{DP-ZOPO} & 35.15 & 16.83 & 25.17 & 25.15 \\
        \bottomrule
    \end{tabular}
\end{table}

\textbf{Translation Task:} We fine-tune OPT-1.3b and OPT-2.7b on Europarl. As shown in Table \ref{table_translation}, DP-ZOPO has better results than FT-prefix, DPZero, DP-ZOSO and approaches FT. However, all DP fine-tuning (first-order and zeroth-order) methods are not as good as ICL. We attribute this situation to two reasons: 1. The OPT model is already well-trained, 2. The injection of differential privacy noise is large, although DP-ZOPO performs well among all fine-tuning methods.

\textbf{Summarization Task:} We fine-tune OPT-1.3b on XSum and CNN/DM with $\epsilon=4$. As shown in Table \ref{table_summarization}, models fine-tuned with DP-ZOSO and DP-ZOPO have metrics close to FT while the existing DP zeroth-order method (DPZero) fails to converge the model. DP-ZOPO outperforms ICL, FT-prefix, DPZero and DP-ZOSO. In XSum, DP-ZOPO outperforms FT-prefix by 3.21 on R1, 1.35 on R2, 2.44 on RL, 2.4 on R-SUM and approaches FT with much less GPU consumption. In CNN/DM, DP-ZOPO outperforms FT-prefix by 3.86 on R1, 1.94 on R2, 2.58 on RL and 2.6 on R-SUM.

\section{Conclusion}
\label{sec:conclusion}
We have provided both theoretical and empirical studies among all concurrent works. Detailed theoretical proof of both privacy and convergence for both our algorithms are provided. We have shown that the stagewise DP zeroth-order method with pruning is memory-efficient and can effectively optimize large LMs across many tasks and scales. Further experiments suggest that the optimal pruning rate varies with the privacy budget $\epsilon$. As a limitation, we did not explore the optimal selection of pruning rates for different tasks and different models, which may be regarded as hyperparameter fine-tuning. Moreover, the adaptive choice of optimal pruning rate across stages and the number of stages is worth studying in the future, as are more effective pruning algorithms and pruning strategies.

\section*{Acknowledgment}
We would like to thank the anonymous reviewers for their helpful comments. This work was supported by the National Key Research and Development Program of China (2021YFB3100300), Nature Science Foundation of China (U20A20178 and 62072395), “Pioneer” and “Leading Goose” R$\And$D Program of Zhejiang Province (2024C01169), Nature Science Foundation of China (62206207).

\bibliographystyle{ieeetran.bst}
\bibliography{bib}

\appendix
\section{Appendix}
In Sec \ref{sec proof}, we provide the proof for the oracle complexity of both DP-ZOSO and DP-ZOPO. In Sec \ref{sec exp}, we will provide the experiment settings and more experiment results in Sec \ref{sec more result}.

\subsection{Complete Proof}\label{sec proof}
In this section, we will provide detialed proof for Lemma \ref{lemma var}, \ref{lemma conver} and Theorem \ref{theo stage}, \ref{theo total}.
\begin{lemma}[Restatation of Lemma \ref{lemma var}]
    In DP-ZOSO, under assumption \ref{assum 1} and the dimension of the parameters to be updated is $d$ with clipping bound $C$. The error between the gradient $\nabla F_{\beta}(\boldsymbol{\theta})$ on the total dataset and the actual update gradient $\nabla\hat{f}_{\beta}(\boldsymbol{\theta}_t,\xi_{t+1})=\frac{1}{P}\sum_{p=1}^{P}\frac{1}{m}\sum_{i=1}^{m}\mathrm{CLIP}\left(\nabla f_{\beta_s}^{p}(\boldsymbol{\theta}_t,x_{t+1}^{i})\right)+\mathbf{z}_t^p$ is bounded by:
    \begin{equation}
        \begin{split}
            \mathbb{E}[||\nabla &F_{\beta}(\boldsymbol{\theta}_{t})-\nabla\hat{f}_{\beta}(\boldsymbol{\theta}_t,\xi_{t+1})||^2] \\
            \leq&\frac{dc_2^2C^2 P T\log(P/\delta)}{\epsilon^2n^2}+\frac{8dC^2}{ePm}+\frac{64d\beta^2L^4}{Pm}+\frac{\gamma^2}{Pm}.
        \end{split}
    \end{equation}
\end{lemma}
\begin{proof}[Proof of lemma \ref{lemma var}]
Due to the introduction of DP noise, we have
\begin{equation}\label{proof 1}
    \begin{split}
        &\mathbb{E}\left[\left\|\nabla F_{\beta_s}(\boldsymbol{\theta}_{t})-\nabla\hat{f}_{\beta_s}(\boldsymbol{\theta}_t,\xi_{t+1})\right\|^2\right] \\
        =&\footnotesize\mathbb{E}\left[\left\|\nabla F_{\beta_s}(\boldsymbol{\theta}_{t})-\frac{1}{P}\sum_{p=1}^{P}\frac{1}{m}\sum_{i=1}^{m}\mathrm{CLIP}\left(\nabla f_{\beta_s}^{p}(\boldsymbol{\theta}_t,x_{t+1}^{i})\right)+\mathbf{z}_t^p\right\|^2\right] \\
        =&\small\frac{1}{P^2}\sum_{p=1}^{P}\mathbb{E}\left[\left\|\nabla F_{\beta_s}(\boldsymbol{\theta}_{t})-\frac{1}{m}\sum_{i=1}^{m}\mathrm{CLIP}\left(\nabla f_{\beta_s}^{p}(\boldsymbol{\theta}_t,x_{t+1}^{i})\right)\right\|^2\right] \\
        &+\frac{dc_2^2C^2 P T\log(P/\delta)}{\epsilon^2n^2}  \\
        \small=&\small\left\| \mathbb{E}_{\mathbf{v}}\left[\nabla F_{\beta_s}(\boldsymbol{\theta}_{t})\right]-\frac{1}{P}\sum_{p=1}^{P}\frac{1}{m}\sum_{i=1}^{m}\mathrm{CLIP}\left(\nabla f_{\beta_s}^{p}(\boldsymbol{\theta}_t,x_{t+1}^{i})\right)\right\|^2 \\
        &+\frac{dc_2^2C^2 P T\log(P/\delta)}{\epsilon^2n^2}. \\
    \end{split}
\end{equation}
Using the elementary inequality $(a-b)^2\leq2a^2+2b^2$, we have
\begin{equation}
    \begin{split}
        &\mathbb{E}\left[\left\|\frac{1}{2\beta}\left(f(\boldsymbol{\theta}_{t-1}+\beta\mathbf{v},x)-f(\boldsymbol{\theta}_{t-1}-\beta\mathbf{v},x)\right)\right\|^2\right]  \\
        =&\frac{1}{4\beta^2}\mathbb{E}[|f(\boldsymbol{\theta}_{t-1}+\beta\mathbf{v},x)-\mathbb{E}_{\mathbf{v}}\left[f(\boldsymbol{\theta}_{t-1}+\beta\mathbf{v},x)\right]\\
        &+\mathbb{E}_{\mathbf{v}}\left[f(\boldsymbol{\theta}_{t-1}+\beta\mathbf{v},x)\right]-f(\boldsymbol{\theta}_{t-1}-\beta\mathbf{v},x)|^2] \\
        \leq&\frac{1}{2\beta^2}\mathbb{E}\left[\left\|f(\boldsymbol{\theta}_{t-1}+\beta\mathbf{v},x)-\mathbb{E}_{\mathbf{v}}\left[f(\boldsymbol{\theta}_{t-1}+\beta\mathbf{v},x)\right]\right\|^2\right] \\
        &+\frac{1}{2\beta^2}\mathbb{E}\left[\left\|\mathbb{E}_{\mathbf{v}}\left[f(\boldsymbol{\theta}_{t-1}+\beta\mathbf{v},x)\right]-f(\boldsymbol{\theta}_{t-1}-\beta\mathbf{v},x)\right\|^2\right].
    \end{split}
\end{equation}
Since $\mathbf{v}$ has a symmetric distribution around the origin, we have
\begin{equation}
    \begin{split}
        &\mathbb{E}\left[\left\|f(\boldsymbol{\theta}_{t-1}+\beta\mathbf{v},x)-\mathbb{E}_{\mathbf{v}}\left[f(\boldsymbol{\theta}_{t-1}+\beta\mathbf{v},x)\right]\right\|^2\right] \\
        =&\mathbb{E}\left[\left\|f(\boldsymbol{\theta}_{t-1}-\beta\mathbf{v},x)-\mathbb{E}_{\mathbf{v}}\left[f(\boldsymbol{\theta}_{t-1}+\beta\mathbf{v},x)\right]\right\|^2\right] .
    \end{split}
\end{equation}
Define $h(\boldsymbol{\theta})=f(\boldsymbol{\theta}+\beta\mathbf{v})$. Since $f$ is $L$-Lipschitz and $\mathbf{v}\in\mathcal{R}^d$ is sampled from standard Gaussian distribution. Then, by [\cite{wainwright2019high} Proposition 3.2], we have
\begin{equation}
    \mathbb{P}\left(\left|h(\boldsymbol{\theta})-\mathbb{E}[h(\boldsymbol{\theta})]\right|\geq c\right) \leq 2 e^{-\frac{c^2}{2\beta^2L^2}}.
\end{equation}
Then, we have
\begin{equation}
    \begin{split}
        \small\mathbb{E}&\small\left[\left(h(\boldsymbol{\theta})-\mathbb{E}[h(\boldsymbol{\theta})]\right)^2\right]=\int_{0}^{+\infty}\mathbb{P}\left(\left|h(\boldsymbol{\theta})-\mathbb{E}[h(\boldsymbol{\theta})]\right|^2\geq c\right) dc \\
        =&\int_{0}^{+\infty}\mathbb{P}\left(\left|h(\boldsymbol{\theta})-\mathbb{E}[h(\boldsymbol{\theta})]\right|\geq \sqrt{c}\right) dc \leq 2\int_{0}^{+\infty}e^{-\frac{c}{2\beta^2 L^2}} dc \\
        =& 4\beta^2L^2.
    \end{split}
\end{equation}
By the definition of $h$, we have
\begin{equation}
    \begin{split}
        &\mathbb{E}\left[\left\|\frac{1}{2\beta}\left(f(\boldsymbol{\theta}_{t-1}+\beta\mathbf{v},x)-f(\boldsymbol{\theta}_{t-1}-\beta\mathbf{v},x)\right)\right\|^2\right] \leq 4L^4.\\
    \end{split}
\end{equation}
Furthermore, we can obtain that
\begin{equation}
    \begin{split}
        &\mathbb{E}\left[\left\|\nabla f_{\beta}(\boldsymbol{\theta})-\mathbb{E}_{\mathbf{v}}\left[\nabla f_{\beta}(\boldsymbol{\theta})\right]\right\|^2\right] \\
        \leq&2d\mathbb{E}\left[\left\|\frac{1}{2\beta}\left(f(\boldsymbol{\theta}+\beta\mathbf{v})-\mathbb{E}_{\mathbf{v}}\left[f(\boldsymbol{\theta}+\beta\mathbf{v})\right]\right)\right\|^2\right] \\
        &+2d\mathbb{E}\left[\left\|\frac{1}{2\beta}\left(f(\boldsymbol{\theta}-\beta\mathbf{v})-\mathbb{E}_{\mathbf{v}}\left[f(\boldsymbol{\theta}-\beta\mathbf{v})\right]\right)\right\|^2\right] \\
        \leq& 64d\beta^2L^4.
    \end{split}
\end{equation}
Following assumption \ref{assum 1}, we can derive that
\begin{equation}\label{proof 2}
    \begin{split}
        &\small\left\| \mathbb{E}_{\mathbf{v}}\left[\nabla F_{\beta_s}(\boldsymbol{\theta}_{t})\right]-\frac{1}{P}\sum_{p=1}^{P}\frac{1}{m}\sum_{i=1}^{m}\mathrm{CLIP}\left(\nabla f_{\beta_s}^{p}(\boldsymbol{\theta}_t,x_{t+1}^{i})\right)\right\|^2  \\
        \leq&\frac{1}{P^2m^2}\sum_{p=1}^{P}\sum_{i=1}^{m}\left\| \nabla F_{\beta_s}^{p}(\boldsymbol{\theta}_{t})-\mathrm{CLIP}\left(\nabla f_{\beta_s}^{p}(\boldsymbol{\theta}_t,x_{t+1}^{i})\right)\right\|^2 \\
        &+\frac{64d \beta_s^2 L^4}{Pm} \\
        \leq&\small\frac{1}{P^2m^2}\sum_{p=1}^{P}\sum_{i=1}^{m}\left\| \nabla f_{\beta_s}^{p}(\boldsymbol{\theta}_t,x_{t+1}^{i})-\mathrm{CLIP}\left(\nabla f_{\beta_s}^{p}(\boldsymbol{\theta}_t,x_{t+1}
        ^{i})\right)\right\|^2 \\
        &+\frac{64d \beta_s^2 L^4}{Pm} +\frac{\gamma^2}{Pm} .\\
    \end{split}
\end{equation}
Furthermore, we can bound the possibility of clipping happens
\begin{equation}
    \begin{split}
        \small\mathbb{P}\left(\frac{1}{2\beta}\left|f(\boldsymbol{\theta}_{t-1}+\beta\mathbf{v},x)-f(\boldsymbol{\theta}_{t-1}-\beta\mathbf{v},x)\right|\geq C\right)\leq 2e^{-\frac{C^2}{2L^2}}.
    \end{split}
\end{equation}
We define $Q_{t+1}^{i}$ to be the event that the clipping does not happen at iteration $t+1$ for sample $x_{t+1}^i$, and $\overline{Q_{t+1}^i}$ to be the event that the clipping does happen. When event $Q_{t+1}^{i}$ happens, clipping does not happen at iteration $t+1$ for sample $x_{t+1}^i$ such that $\small\left\|\nabla f_{\beta_s}(\boldsymbol{\theta}_t,x_{t+1}^{i})-\mathrm{CLIP}\left(\nabla f_{\beta_s}(\boldsymbol{\theta}_t,x_{t+1}^{i})\right)\right\|^2=0$, we can obtain that
\begin{equation}
    \begin{split}
        &\mathbb{E}\left[\left\|\nabla f_{\beta_s}(\boldsymbol{\theta}_t,x_{t+1}^i)-\mathrm{CLIP}\left(\nabla f_{\beta_s}(\boldsymbol{\theta}_t,x_{t+1}^i)\right)\right\|^2\right] \\
        =&\left\|\nabla f_{\beta_s}(\boldsymbol{\theta}_t,x_{t+1}^i)-C\right\|^2\mathbb{P}(\overline{Q_{t+1}^{i}}) \\
        \leq&\left(2dC^2+4dL^2\right)\cdot\exp\left(-\frac{C^2}{2L^2}\right). \\
    \end{split}
\end{equation}
By setting $C^2=2L^2$, we have
\begin{equation}\label{proof 3}
    \begin{split}
        &\mathbb{E}\left[\left\|\nabla f_{\beta_s}(\boldsymbol{\theta}_t,x_{t+1}^i)-\mathrm{CLIP}\left(\nabla f_{\beta_s}(\boldsymbol{\theta}_t,x_{t+1}^i)\right)\right\|^2\right] \leq\frac{8dC^2}{e} .\\
    \end{split}
\end{equation}
Combining inequality \ref{proof 1}, \ref{proof 2} and \ref{proof 3}, we can bound the following term
\begin{equation}
    \begin{split}
        \small\mathbb{E}[||\nabla &F_{\beta_s}(\boldsymbol{\theta}_{t})-\nabla\hat{f}_{\beta_s}(\boldsymbol{\theta}_t,\xi_{t+1})||^2] \\
        \leq&\frac{dc_2^2C^2 P T\log(P/\delta)}{\epsilon^2n^2}+\frac{64d\beta_s^2L^4}{Pm}+\frac{8dC^2}{ePm}+\frac{\gamma^2}{Pm}.
    \end{split}
\end{equation}
\end{proof}

\begin{lemma}[Restatation of Lemma \ref{lemma conver}]
    Under Assumption \ref{assum 1} and $f(\boldsymbol{\theta},x)$ is $\rho$-weakly-convex of $\boldsymbol{\theta}$, by applying DP-ZOO (Algorithm \ref{alg:2}) with $\eta_s\leq\frac{\beta_s}{\sqrt{d}L}$, for any $\boldsymbol{\theta}\in\mathcal{R}^d$, we have
\begin{equation}
    \begin{split}
        \mathbb{E}&[\phi_{s}(\hat{\boldsymbol{\theta}}_{T_s})-\phi_{s}(\boldsymbol{\theta})]
        \leq\left(\frac{1}{2\eta_s T_s}+\frac{1}{2T_s\lambda}\right)\left\|\boldsymbol{\theta}^{s-1}-\boldsymbol{\theta}\right\|^2 \\
        &+\eta_s \left(\frac{dc_2^2C^2 P T\log(P/\delta)}{\epsilon^2n^2}+\frac{64d\beta_s^2L^4}{Pm}+\frac{8dC^2}{ePm}+\frac{\gamma^2}{Pm}\right).\\
    \end{split}
\end{equation}
\end{lemma}
\textbf{Proof for Lemma \ref{lemma conver}:} 
Recall that $\phi_{s}(\boldsymbol{\theta})=f_{\beta_s}(\boldsymbol{\theta})+\frac{1}{2\lambda}\left\|\boldsymbol{\theta}_t-\boldsymbol{\theta}^s\right\|^2$, let $F_{\beta_s}(\boldsymbol{\theta})=f_{\beta_s}(\boldsymbol{\theta},\mathcal{D})$, $r_s(\boldsymbol{\theta})=\frac{1}{2\lambda}\left\|\boldsymbol{\theta}-\boldsymbol{\theta}^s\right\|^2$. Due to the $\rho$-weak-convexity of $f(\boldsymbol{\theta})$, the $\frac{1}{\lambda}$-strong-convexity of $r_s(\boldsymbol{\theta})$ and the $\frac{\sqrt{d}L}{\beta_s}$-smoothness of $f_{\beta_s}(\boldsymbol{\theta})$, we have the following three inequality
\begin{equation}
    \begin{split}
        & F_{\beta_s}(\boldsymbol{\theta}) \geq F_{\beta_s}(\boldsymbol{\theta}_t) +\langle \nabla F_{\beta_s}(\boldsymbol{\theta}_t),\boldsymbol{\theta}-\boldsymbol{\theta}_t \rangle - \frac{\rho}{2} \left\|\boldsymbol{\theta}_{t}-\boldsymbol{\theta}\right\|^2 \\
        & r_s(\boldsymbol{\theta}) \geq r_s(\boldsymbol{\theta}_{t+1}) + \langle\partial r_s(\boldsymbol{\theta}_{t+1}),\boldsymbol{\theta}-\boldsymbol{\theta}_{t+1}\rangle +\frac{1}{2\lambda}\left\|\boldsymbol{\theta}-\boldsymbol{\theta}_{t+1}\right\|^2 \\
        & \footnotesize F_{\beta_s}(\boldsymbol{\theta}_t) \geq F_{\beta_s}(\boldsymbol{\theta}_{t+1}) -\langle \nabla F_{\beta_s}(\boldsymbol{\theta}_t),\boldsymbol{\theta}_{t+1}-\boldsymbol{\theta}_t \rangle - \frac{\sqrt{d}L}{2\beta_s}\left\|\boldsymbol{\theta}_t-\boldsymbol{\theta}_{t+1}\right\|^2. \\
    \end{split}
\end{equation}
Combing them together, we have 
\begin{equation}
    \begin{split}
        &F_{\beta_s}(\boldsymbol{\theta}_{t+1})+r_s(\boldsymbol{\theta}_{t+1})-(F_{\beta_s}(\boldsymbol{\theta})+r_s(\boldsymbol{\theta})) \\
        \leq &\langle \nabla F_{\beta_s}(\boldsymbol{\theta}_t)+\partial r_s(\boldsymbol{\theta}_{t+1}),\boldsymbol{\theta}_{t+1}-\boldsymbol{\theta}\rangle \\ &+\frac{\rho}{2}\left\|\boldsymbol{\theta}_{t}-\boldsymbol{\theta}\right\|^2+\frac{\sqrt{d}L}{2\beta_s}\left\|\boldsymbol{\theta}_t-\boldsymbol{\theta}_{t+1}\right\|^2-\frac{1}{2\lambda}\left\|\boldsymbol{\theta}-\boldsymbol{\theta}_{t+1}\right\|^2 .\label{1}
    \end{split}
\end{equation}
If we set the gradient in $\boldsymbol{\theta}_{t+1}$ to 0, there exists $\partial r_s(\boldsymbol{\theta}_{t+1})$ that
\begin{equation}
    \partial r_s(\boldsymbol{\theta}_{t+1})=\frac{1}{\eta_s}(\boldsymbol{\theta}_{t}-\boldsymbol{\theta}_{t+1})- \nabla \hat{f}_{\beta_s}(\boldsymbol{\theta}_t,\xi_{t+1}).
\end{equation}
where $\nabla\hat{f}_{\beta_s}(\boldsymbol{\theta}_t,\xi_{t+1})=\frac{1}{P}\sum_{p=1}^{P}\frac{1}{m}\sum_{i=1}^{m}\mathrm{CLIP}($$\nabla f_{\beta_s}^{p}(\boldsymbol{\theta}_t,$ $x_{t+1}^{i}$ $))$ $+\mathbf{z}_t^p$ and $\mathbf{z}_t^p=z_t^p\cdot\mathbf{v}_t^p\left(z_t^p\sim N(0,\sigma^2 C^2)\right)$. Plugging the above equation into \ref{1} and setting $\hat{\boldsymbol{\theta}}_{t+1}$ be the updated parameter on dataset $\mathcal{D}$ without DP noise in iteration $t+1$. Since $\eta$ is decreasing and $\beta$ is increasing, $\eta_s\leq\frac{\beta_s}{\sqrt{d}L}$ holds true for all $s\in[0,S]$ when $\eta_1\leq\frac{\beta_1}{\sqrt{d}L}$, by taking $\eta_s\leq\frac{\beta_s}{\sqrt{d}L}$, we have
\begin{small}
\begin{equation}
    \begin{aligned}
        &F_{\beta_s}(\boldsymbol{\theta}_{t+1})+r_s(\boldsymbol{\theta}_{t+1})-(F_{\beta_s}(\boldsymbol{\theta})+r_s(\boldsymbol{\theta})) \\
        \leq&\langle\nabla F_{\beta_s}(\boldsymbol{\theta}_{t})-\nabla\hat{f}_{\beta_s}(\boldsymbol{\theta}_t,\xi_{t+1}),\boldsymbol{\theta}_{t+1}-\boldsymbol{\theta}\rangle\\
        &+\langle\frac{1}{\eta_s}(\boldsymbol{\theta}_t-\boldsymbol{\theta}_{t+1}),\boldsymbol{\theta}_{t+1}-\boldsymbol{\theta}\rangle +\frac{\rho}{2}\left\|\boldsymbol{\theta}_{t}-\boldsymbol{\theta}\right\|^2 \\
        &+\frac{\sqrt{d}L}{2\beta_s}\left\|\boldsymbol{\theta}_t-\boldsymbol{\theta}_{t+1}\right\|^2-\frac{1}{2\lambda}\left\|\boldsymbol{\theta}-\boldsymbol{\theta}_{t+1}\right\|^2 \\
        =&\langle\nabla F_{\beta_s}(\boldsymbol{\theta}_{t})-\nabla\hat{f}_{\beta_s}(\boldsymbol{\theta}_t,\xi_{t+1}),\boldsymbol{\theta}_{t+1}-\hat{\boldsymbol{\theta}}_{t+1}+\hat{\boldsymbol{\theta}}_{t+1}-\boldsymbol{\theta}\rangle \\
        &+\left(\frac{1}{2\eta_s}+\frac{\rho}{2}\right)\left\|\boldsymbol{\theta}_{t}-\boldsymbol{\theta}\right\|^2 -\left(\frac{1}{2\eta_s}+\frac{1}{2\lambda}\right)\left\|\boldsymbol{\theta}-\boldsymbol{\theta}_{t+1}\right\|^2\\
        \leq&\langle\nabla F_{\beta_s}(\boldsymbol{\theta}_{t})-\nabla\hat{f}_{\beta_s}(\boldsymbol{\theta}_t,\xi_{t+1}),\hat{\boldsymbol{\theta}}_{t+1}-\boldsymbol{\theta}\rangle \\
        &+\eta_s\left\|\nabla F_{\beta_s}(\boldsymbol{\theta}_{t})-\nabla\hat{f}_{\beta_s}(\boldsymbol{\theta}_t,\xi_{t+1})\right\|^2\\
        &+\left(\frac{1}{2\eta_s}+\frac{\rho}{2}\right)\left\|\boldsymbol{\theta}_{t}-\boldsymbol{\theta}\right\|^2 -\left(\frac{1}{2\eta_s}+\frac{1}{2\lambda}\right)\left\|\boldsymbol{\theta}-\boldsymbol{\theta}_{t+1}\right\|^2.\\
    \end{aligned}
\end{equation}
\end{small}
Taking expectation on both sides, we have
\begin{equation}
    \begin{split}
        &\small\mathbb{E}\left[\phi_{s}(\boldsymbol{\theta}_{t+1})-\phi_{s}(\boldsymbol{\theta})\right] 
        \leq\eta_s\mathbb{E}\left[\left\|\nabla F_{\beta_s}(\boldsymbol{\theta}_{t})-\nabla\hat{f}_{\beta_s}(\boldsymbol{\theta}_t,\xi_{t+1})\right\|^2\right] \\
        &\small+\left(\frac{1}{2\eta_s}+\frac{\rho}{2}\right)\left\|\boldsymbol{\theta}_{t}-\boldsymbol{\theta}\right\|^2-\left(\frac{1}{2\eta_s}+\frac{1}{2\lambda}\right)\left\|\boldsymbol{\theta}-\boldsymbol{\theta}_{t+1}\right\|^2 .\\
    \end{split}
\end{equation}
Following Lemma \ref{lemma var} and by taking summation of the above inequality from $t=0$ to $T_s-1$, $\lambda< 1/\rho$, we have
\begin{equation}
    \begin{split}
        &\sum_{t=0}^{T_s-1}\phi_{s}(\boldsymbol{\theta}_{t+1})-\phi_{s}(\boldsymbol{\theta})
        \leq\left(\frac{1}{2\eta_s}+\frac{\rho}{2}\right)\left\|\boldsymbol{\theta}^{s-1}-\boldsymbol{\theta}\right\|^2 \\
        &+\eta_s T_s\left(\frac{dc_2^2C^2 P T\log(P/\delta)}{\epsilon^2n^2}+\frac{64d\beta_s^2L^4}{Pm}+\frac{8dC^2}{ePm}+\frac{\gamma^2}{Pm}\right) \\
        &-\left(\frac{1}{2\eta_s}+\frac{1}{2\lambda}\right)\left\|\boldsymbol{\theta}-\boldsymbol{\theta}_{T_s}\right\|^2.
    \end{split}
\end{equation}
By employing Jensens' inequality, denoting the output of stage $s$ by $\boldsymbol{\theta}^s=\hat{\boldsymbol{\theta}}_{T_s}=\frac{1}{T_s}\sum_{t=1}^{T_s}\boldsymbol{\theta}_{t}$, we can obtain that
\begin{equation}
    \begin{split}
        &\phi_{s}(\hat{\boldsymbol{\theta}}_{T_s})-\phi_{s}(\boldsymbol{\theta})
        \leq\left(\frac{1}{2\eta_s T_s}+\frac{1}{2T_s\lambda}\right)\left\|\boldsymbol{\theta}^{s-1}-\boldsymbol{\theta}\right\|^2 \\
        &+\eta_s \left(\frac{dc_2^2C^2 P T\log(P/\delta)}{\epsilon^2n^2}+\frac{64d\beta_s^2L^4}{Pm}+\frac{\gamma^2}{Pm}+\frac{8dC^2}{ePm}\right).\\
    \end{split}
\end{equation}
\begin{theorem}[Restatation of Theorem \ref{theo stage}]
    In DP-ZOSO, suppose assumption \ref{assum 1} holds and the loss function $f(\boldsymbol{\theta},x)$ is $\rho$-weakly-convex of $\boldsymbol{\theta}$. Then by setting $\eta_s=\alpha_{s}\cdot\min \{\frac{Pm}{7\gamma^2},\frac{Pme}{56dC^2}$,$\frac{\epsilon^2 n^2}{7dc_2^2C^2PT\log(P/\delta)}$, 
    $\frac{Pm}{448d\beta_S^2L^4}\}$ and $\lambda=\frac{7}{2\mu}, \eta_s T_s=\frac{7}{2\mu}$, after $S=\lceil\log(\alpha_0/\alpha)\rceil$ stages, we have
    \begin{equation}
        \phi_{s}(\boldsymbol{\theta}^{S})-\phi_{s}(\boldsymbol{\theta}^*)\leq \alpha.
    \end{equation}
    The total ZO oracle complexity is 
    \begin{equation}
        \mathcal{O}\left(\left(\gamma^2+dC^2+d\beta_S^2L^4+\frac{dC^2P^2m\log(P/\delta)}{\epsilon^2 n^2}\right)\cdot\frac{1}{\mu\alpha}\right).
    \end{equation}
\end{theorem}
\begin{proof}[Proof of Theorem \ref{theo stage}]
By taking $\boldsymbol{\theta}=\boldsymbol{\theta}^*$, we have
\begin{footnotesize}
    \begin{equation}
        \begin{aligned}
            f_{\beta_s}(\hat{\boldsymbol{\theta}}_{T_s})&-f_{\beta_s}(\boldsymbol{\theta}^*) \leq\left(\frac{1}{2\lambda}+\frac{1}{2\eta_s T_s}+\frac{1}{2T_s\lambda}\right)\left\|\boldsymbol{\theta}^{s-1}-\boldsymbol{\theta}^*\right\|^2 \\
            &+\eta_s \left(\frac{dc_2^2C^2 P T\log(P/\delta)}{\epsilon^2n^2}+\frac{64d\beta_s^2L^4}{Pm}+\frac{\gamma^2}{Pm}+\frac{8dC^2}{ePm}\right).\\
        \end{aligned}
    \end{equation}
\end{footnotesize}

Since $f_{\beta}$ satisfies $\mu$-PL condition, we will prove by induction that $f_{\beta_s}(\boldsymbol{\theta}^s)-f_{\beta_s}(\boldsymbol{\theta}^*)\leq\alpha_s$, where $\alpha_s=\alpha_0/2^s$, which is true for $s = 0$, we have
\begin{small}
\begin{equation}
    \begin{aligned}
        &f_{\beta_s}(\boldsymbol{\theta}^{s})-f_{\beta_s}(\boldsymbol{\theta}^*)\\
        \leq&\left(\frac{1}{4\lambda\mu}+\frac{1}{4\eta_s T_s\mu}+\frac{1}{4T_s\lambda\mu}\right)\left(f_{\beta_{s-1}}(\boldsymbol{\theta}^{s-1})-f_{\beta_{s-1}}(\boldsymbol{\theta}^*)\right) \\
        &+\eta_s \left(\frac{dc_2^2C^2 P T\log(P/\delta)}{\epsilon^2n^2}+\frac{64d\beta_s^2L^4}{Pm}+\frac{\gamma^2}{Pm}+\frac{8dC^2}{ePm}\right) \\
        \leq&\left(\frac{1}{4\lambda\mu}+\frac{1}{4\eta_s T_s\mu}+\frac{1}{4T_s\lambda\mu}\right)\alpha_{s-1} \\
        &+\eta_s \left(\frac{dc_2^2C^2 P T\log(P/\delta)}{\epsilon^2n^2}+\frac{64d\beta_S^2L^4}{Pm}+\frac{\gamma^2}{Pm}+\frac{8dC^2}{ePm}\right). \\
    \end{aligned}
\end{equation}
\end{small}
By setting $\eta_s=\alpha_{s}\cdot\min\{\frac{Pm}{7\gamma^2}$, $\frac{Pme}{56dC^2}$, $\frac{\epsilon^2 n^2}{7dc_2^2C^2PT\log(P/\delta)}$, $\frac{Pm}{448d\beta_S^2L^4}\}$ and $\lambda=\frac{7}{2\mu}, \eta_sT_s=\frac{7}{2\mu}$, we have 
\begin{equation}
    f_{\beta_s}(\boldsymbol{\theta}^{s})-f_{\beta_s}(\boldsymbol{\theta}^*)\leq \alpha_s.
\end{equation}
By induction, after $S=\lceil\log(\alpha_0/\alpha_S)\rceil$ stages, we have
\begin{equation}
    f_{\beta_S}(\boldsymbol{\theta}^{S})-f_{\beta_S}(\boldsymbol{\theta}^*)\leq \alpha_S.
\end{equation}
The total ZO oracle complexity of DP-ZOSO is 
\begin{equation}
    \mathcal{O}\left(\left(\gamma^2+dC^2+d\beta_S^2L^4+\frac{dC^2P^2m\log(P/\delta)}{\epsilon^2 n^2}\right)\cdot\frac{1}{\mu\alpha}\right).
\end{equation}
\end{proof}
\begin{lemma}\label{lemma pruning}
    In DP-ZOPO, under assumption \ref{assum 1}, the dimension of parameters to be fine-tuned in stage $s$ is reduced to $d*r_s$. The error between the gradient $\nabla_\mathbb{V} F_{\beta}(\boldsymbol{\theta})$ on the total dataset and the actual update gradient $\nabla_\mathbb{V}\hat{f}_{\beta}(\boldsymbol{\theta}_t,\xi_{t+1})$ is bounded by:
    \begin{small}
    \begin{equation}
        \begin{aligned}
            \mathbb{E}[&||\nabla_\mathbb{V} F_{\beta}(\boldsymbol{\theta}_{t})-\nabla_\mathbb{V}\hat{f}_{\beta}(\boldsymbol{\theta}_t,\xi_{t+1})||^2] \\
            \leq&\frac{dr_sc_2^2C^2 P T\log(P/\delta)}{\epsilon^2n^2}+\frac{8dr_sC^2}{ePm}+\frac{64dr_s\beta^2L^4}{Pm}+\frac{\gamma^2}{Pm}.
        \end{aligned}
    \end{equation}
    \end{small}
\end{lemma}
\begin{proof}[Proof of Lemma \ref{lemma pruning}]
    In DP-ZOPO, the dimension of parameters to be fine-tuned in stage $s$ is $d*r_s(r_s\leq 1)$. First, the scale of DP noise introduced is reduced by the current pruning rate $r_s$. Then, the $d$ in the clipping error and the zeroth-order estimation error(second and third terms) is replaced with $d*r_s$ since the $l2$-norm of the direction vector $\left\|\mathbf{v}\right\|(\mathbf{v}\sim\mathbb{V})$ is reduced to $d*r_s$.
\end{proof}
\begin{theorem}[Restatation of Theorem \ref{theo total}]
    In DP-ZOPO, suppose assumption \ref{assum 1} holds and loss function $f(\boldsymbol{\theta},x)$ is $\rho$-weakly-convex of $\boldsymbol{\theta}$. The dimensionality of parameters to be updated in stage $s$ is $d\cdot r_s$. Then by setting $\eta_s=\alpha_{s}\cdot\min\left\{\frac{Pm}{7\gamma^2},\frac{Pme}{56dr_SC^2},\frac{\epsilon^2 n^2}{7dr_Sc_2^2C^2PT\log(P/\delta)},\frac{Pm}{448dr_S\beta_S^2L^4}\right\}$ and $\lambda=\frac{7}{2\mu}, \eta_sT_s=\frac{7}{2\mu}$, after $S=\lceil\log(\alpha_0/\alpha)\rceil$ stages, we have
    \begin{equation}
        \phi_{s}(\boldsymbol{\theta}'^{S})-\phi_{s}(\boldsymbol{\theta}^*)\leq \alpha.
    \end{equation}
    The total ZO oracle complexity of DP-ZOPO is 
    \begin{equation}
    \footnotesize\mathcal{O}\left(\left(\gamma^2+rdr_SC^2+rdr_S\beta_S^2L^4+\frac{rdr_SC^2P^2m\log(P/\delta)}{\epsilon^2 n^2}\right)\cdot\frac{1}{\mu\alpha}\right).
    \end{equation}
\end{theorem}
\begin{proof}[Proof of Theorem \ref{theo total}]
Similar in Theorem \ref{theo stage}, we can obtain that
\begin{small}
    \begin{equation}
        \begin{aligned}
            &f_{\beta_s}(\boldsymbol{\theta}^{s})-f_{\beta_s}(\boldsymbol{\theta}^*)\\
            \leq&\left(\frac{1}{4\lambda\mu}+\frac{1}{4\eta_s T_s\mu}+\frac{1}{4T_s\lambda\mu}\right)\alpha_{s-1}+\frac{64\eta_s dr_S\beta_S^2L^4}{Pm} \\
            &+ \frac{\eta_sdr_Sc_2^2C^2 P T\log(P/\delta)}{\epsilon^2n^2}+\frac{8\eta_s dr_SC^2}{ePm}+\frac{\eta_s\gamma^2}{Pm}. \\
        \end{aligned}
    \end{equation}
\end{small}
By setting $\eta_s=\alpha_{s}\cdot\min\{\frac{Pm}{7\gamma^2}$, $\frac{Pme}{56dr_SC^2}$, $\frac{\epsilon^2 n^2}{7dr_Sc_2^2C^2PT\log(P/\delta)}$, $\frac{Pm}{448dr_S\beta_S^2L^4}\}$ and $\lambda=\frac{7}{2\mu}, \eta_sT_s=\frac{7}{2\mu}$, we have 
\begin{equation}
    f_{\beta_s}(\boldsymbol{\theta}^{s})-f_{\beta_s}(\boldsymbol{\theta}^*)\leq \alpha_s.
\end{equation}
By induction, after $S=\lceil\log(\alpha_0/\alpha_S)\rceil$ stages, we have
\begin{equation}
    f_{\beta_S}(\boldsymbol{\theta}^{S})-f_{\beta_S}(\boldsymbol{\theta}^*)\leq \alpha_S.
\end{equation}
The total ZO oracle complexity of DP-ZOPO is 
\begin{equation}
\small\mathcal{O}\left(\left(\gamma^2+dr_SC^2+dr_S\beta_S^2L^4+\frac{dr_SC^2P^2m\log(P/\delta)}{\epsilon^2 n^2}\right)\cdot\frac{1}{\mu\alpha}\right).
\end{equation}
\end{proof}
\subsection{Hyperparameters}\label{sec exp}

In this section, we will provide the experiment settings of RoBERTa-large, OPT-1.3b and OPT-2.7b. We use 6000 steps of training for classification and question-answering tasks. However, due to the complexity of text-generation tasks (translation and summarization), we use 11000 steps for training. In the following tables, BS denotes batch size, LR denotes learning rate, regularization denotes the regularization coefficient and $\beta$ denotes the ZO scale parameter.

\begin{table}[ht]
    \centering
    \vspace{1em}
    \caption{The hyperparameters grids used for RoBERTa-large.}
    \label{roberta-hyper}
    \resizebox{0.8\linewidth}{!}{
    \begin{tabular}{l l c}
    \toprule
        Method & Hyper & Values \\
        \hline
        \hline
        \specialrule{0em}{1pt}{1pt}
        \multirow{2}{*}{FT} & BS & 64 \\
        & LR & \{1e-4, 5e-4, 1e-3, 5e-3\} \\
        \hline
        \specialrule{0em}{1pt}{1pt}
        \multirow{3}{*}{FT-prefix} & BS & 64 \\
        & LR & \{1e-2, 3e-2, 5e-2\} \\
        & prefix tokens & 5 \\
        \hline
        \specialrule{0em}{1pt}{1pt}
        \multirow{3}{*}{DPZero} & BS & 64 \\
        & LR & \{2e-5, 1e-5, 5e-6\} \\
        & $\beta$ & 1e-6 \\
        \hline
        \specialrule{0em}{1pt}{1pt}
        \multirow{6}{50pt}{DP-ZOSO DP-ZOPO} & BS & 64 \\
        & LR & \{2e-4, 1e-4, 5e-6\} \\
        & $\beta$ & 1e-6 to 1e-5 \\
        & pruning rate & \{0.5\%,1\%,2\%,4\%,8\%,100\%\} \\
        & stage size & 3 \\
        & regularization & 5e-4 \\
        \bottomrule
    \end{tabular}
    }
    \vspace{2em}
\end{table}

\begin{table}[ht]
    \centering
    \vspace{1em}
    \caption{The hyperparameters grids used for OPT-1.3b.}
    \label{opt1.3-hyper}
    \resizebox{0.8\linewidth}{!}{
    \begin{tabular}{l l c}
    \toprule
        Method & Hyper & Values \\
        \hline
        \hline
        \specialrule{0em}{1pt}{1pt}
        \multirow{2}{*}{FT} & BS & 8, 16 \\
        & LR & \{1e-6, 5e-6, 1e-5, 5e-5, 1e-4\} \\
        \hline
        \specialrule{0em}{1pt}{1pt}
        \multirow{3}{*}{FT-prefix} & BS & 16 \\
        & LR & \{1e-2, 3e-2, 5e-2\} \\
        & prefix tokens & 5 \\
        \hline
        \specialrule{0em}{1pt}{1pt}
        \multirow{3}{*}{DPZero} & BS & 16, 64 \\
        & LR & \{6e-7, 8e-7, 1e-6, 2e-6\} \\
        & $\beta$ & 1e-6 \\
        \hline
        \specialrule{0em}{1pt}{1pt}
        \multirow{6}{50pt}{DP-ZOSO DP-ZOPO} & BS & 16, 64 \\
        & LR & \{6e-7, 8e-7, 1e-6, 2e-6\} \\
        & $\beta$ & 1e-6 to 1e-5 \\
        & pruning rate & \{0.5\%,1\%,2\%,4\%,8\%,100\%\} \\
        & stage size & 3, 5 \\
        & regularization & 5e-4 \\
        \bottomrule
    \end{tabular}
    }
    \vspace{1em}
\end{table}

\begin{table}[ht]
    \centering
    \vspace{1em}
    \caption{The hyperparameters grids used for OPT-2.7b.}
    \label{opt2.7-hyper}
    \resizebox{0.8\linewidth}{!}{
    \begin{tabular}{l l c}
    \toprule
        Method & Hyper & Values \\
        \hline
        \hline
        \specialrule{0em}{1pt}{1pt}
        \multirow{2}{*}{FT} & BS & 4, 8 \\
        & LR & \{1e-6, 5e-6, 1e-5, 5e-5\} \\
        \hline
        \specialrule{0em}{1pt}{1pt}
        \multirow{3}{*}{FT-prefix} & BS & 4, 8 \\
        & LR & \{1e-3, 5e-3, 1e-2\} \\
        & prefix tokens & 5 \\
        \hline
        \specialrule{0em}{1pt}{1pt}
        \multirow{3}{*}{DPZero} & BS & 16 \\
        & LR & \{6e-7, 8e-7, 1e-6, 2e-6\} \\
        & $\beta$ & 1e-4 \\
        \hline
        \specialrule{0em}{1pt}{1pt}
        \multirow{6}{50pt}{DP-ZOSO DP-ZOPO} & BS & 16 \\
        & LR & \{6e-7, 8e-7, 1e-6, 2e-6\} \\
        & $\beta$ & 1e-6 to 1e-5 \\
        & pruning rate & \{0.5\%,1\%,2\%,4\%,8\%,100\%\} \\
        & stage size & 3, 5 \\
        & regularization & 5e-4 \\
        \bottomrule
    \end{tabular}
    }
    \vspace{1em}
\end{table}


\subsection{More Experiment Results}\label{sec more result}
In this section, we will provide more experiment results that are not included in the main body of this paper. 

\begin{table}[htbp]
\vspace{1em}
\caption{Experiments on RoBERTa-large (350M parameters) with privacy budget $\epsilon=2$.}
\centering
\begin{tabular}{c c c c c c c }
\toprule
\specialrule{0em}{1pt}{1pt}
\multicolumn{2}{c}{Task} & SST-2 & SST-5 & SNLI & MNLI & TREC  \\
\hline
\specialrule{0em}{1pt}{1pt}
\multicolumn{2}{l}{$r=0.5\%$} & 88.5 & 46.5 & 70.3 & 70.1 & 83.4 \\
\multicolumn{2}{l}{$r=1\%$} & 91.7 & 48.8 & 78.4 & 68.1 & 88.6 \\
\multicolumn{2}{l}{$r=2\%$} & 88.4 & 47.1 & 78.2 & 71.4 & 89.6 \\
\multicolumn{2}{l}{$r=4\%$} & 92.2 & 46.6 & 79.5 & 69.8 & 85.0 \\
\multicolumn{2}{l}{$r=8\%$} & 92.5 & 46.6 & 77.3 & 69.3 & 85.4 \\
\multicolumn{2}{l}{$r=100\%$} & 92.9 & 45.8 & 76.8 & 66.3 & 83.4 \\
\bottomrule
\end{tabular}
\end{table}

\begin{table}[htbp]
\vspace{1em}
\caption{Experiments on RoBERTa-large (350M parameters) with privacy budget $\epsilon=4$.}
\centering
\begin{tabular}{c c c c c c c }
\toprule
\specialrule{0em}{1pt}{1pt}
\multicolumn{2}{c}{Task} & SST-2 & SST-5 & SNLI & MNLI & TREC  \\
\hline
\specialrule{0em}{1pt}{1pt}
\multicolumn{2}{l}{$r=0.5\%$} & 92.2 & 50.5 & 78.6 & 71.4 & 81.0 \\
\multicolumn{2}{l}{$r=1\%$} & 92.2 & 47.1 & 80.4 & 64.7 & 88.2 \\
\multicolumn{2}{l}{$r=2\%$}& 91.6 & 47.7 & 80.4 & 71.5 & 87.0 \\
\multicolumn{2}{l}{$r=4\%$}& 93.2 & 47.0 & 78.3 & 69.7 & 86.4 \\
\multicolumn{2}{l}{$r=8\%$} & 92.5 & 48.4 & 78.6 & 69.7 & 84.4 \\
\multicolumn{2}{l}{$r=100\%$}& 92.6 & 42.0 & 78.9 & 67.3 & 83.8 \\
\bottomrule
\end{tabular}
\end{table}

\begin{table}[htbp]
\vspace{1em}
\caption{Experiments on RoBERTa-large (350M parameters) with privacy budget $\epsilon=8$.}
\centering
\begin{tabular}{c c c c c c c }
\toprule
\specialrule{0em}{1pt}{1pt}
\multicolumn{2}{c}{Task} & SST-2 & SST-5 & SNLI & MNLI & TREC  \\
\hline
\specialrule{0em}{1pt}{1pt}
\multicolumn{2}{l}{$r=0.5\%$} & 92.2 & 49.8 & 81.4 & 71.6 & 85.4\\
\multicolumn{2}{l}{$r=1\%$} & 92.2 & 47.4 & 79.2 & 73.1 & 90.8\\
\multicolumn{2}{l}{$r=2\%$} & 91.8 & 49.4 & 78.9 & 68.6 & 85.4\\
\multicolumn{2}{l}{$r=4\%$} & 93.6 & 48.0 & 77.4 & 68.4 & 87.0\\
\multicolumn{2}{l}{$r=8\%$} & 92.4 & 48.0 & 79.0 & 69.7 & 84.4\\
\multicolumn{2}{l}{$r=100\%$} & 93.2 & 41.2 & 76.5 & 69.2 & 84.2\\
\bottomrule
\end{tabular}
\end{table}

\begin{table}[htbp]
\vspace{1.5em}
\caption{Experiments of RoBERTa-large zeroth-order and first-order fine-tuned on SNLI with different pruning rates.}
\raggedright  

\label{table_roberta_firstvszero_pruning}
\vspace{1em}
\begin{subtable}{\linewidth}
\centering
\caption{Zeroth-order}
\begin{tabular}{c c c c c c}
\toprule
\multicolumn{2}{c}{Pruning Rate} & $\epsilon=2$ & $\epsilon=4$ & $\epsilon=8$ & $\epsilon=\infty$ \\
\hline
\hline
\specialrule{0em}{1pt}{1pt}
\multicolumn{2}{c}{$r=0.5\%$} & 70.3 & 78.6 & \textbf{81.4} & \textbf{80.5} \\
\multicolumn{2}{c}{$r=1\%$} & 78.4 & \textbf{80.4} & 79.2 & 80.3 \\
\multicolumn{2}{c}{$r=2\%$} & 78.2 & \textbf{80.4} & 78.9 & 79.6 \\
\multicolumn{2}{c}{$r=4\%$} & \textbf{79.5} & 78.3 & 77.4 & 79.3 \\
\multicolumn{2}{c}{$r=8\%$} & 77.3 & 78.6 & 79.0 & 78.6 \\
\multicolumn{2}{c}{$r=100\%$} & 76.8 & 78.9 & 76.5 & 78.1 \\
\bottomrule
\end{tabular}
\end{subtable}

\vspace{2em}

\begin{subtable}{\linewidth}
\centering
\caption{First-order}
\begin{tabular}{c c c c c c}
\toprule
\multicolumn{2}{c}{Pruning Rate} & $\epsilon=2$ & $\epsilon=4$ & $\epsilon=8$ & $\epsilon=\infty$ \\
\hline
\hline
\specialrule{0em}{1pt}{1pt}
\multicolumn{2}{c}{$r=0.5\%$} & 85.0 & 85.8 & 87.0 & 86.8 \\
\multicolumn{2}{c}{$r=1\%$} & \textbf{87.0} & 86.8 & 86.5 & 85.8 \\
\multicolumn{2}{c}{$r=2\%$} & 86.6 & 86.6 & 86.7 & 88.2 \\
\multicolumn{2}{c}{$r=4\%$} & 82.8 & 85.2 & 86.0 & 87.6 \\
\multicolumn{2}{c}{$r=8\%$} & 84.6 & 86.1 & 86.7 & 86.9 \\
\multicolumn{2}{c}{$r=100\%$} & 86.4 & \textbf{87.0} & \textbf{87.2} & \textbf{90.0} \\
\bottomrule
\end{tabular}
\end{subtable}
\end{table}

\begin{figure}[htb]
    \centering
    \includegraphics[width=\linewidth]{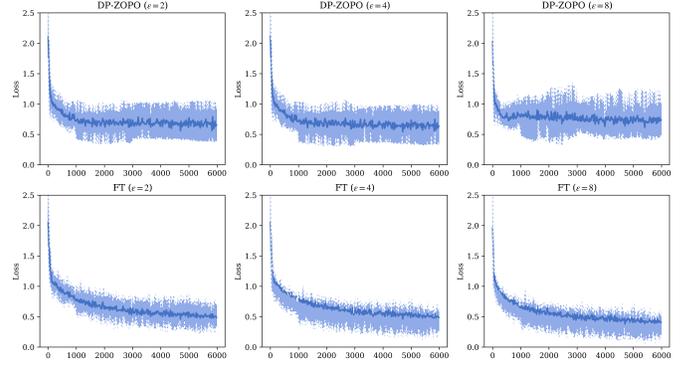}
    \caption{The training loss on SNLI with RoBERTa-large of FT and DP-ZOPO with different privacy budgets ($\epsilon=2,4,8$) that are commonly considered in related works \cite{zhang2024dpzero} \cite{yu2021differentially} \cite{li2021large}. DP-ZOPO with larger privacy budget ($\epsilon=8$) has faster convergence rate. Under the same privacy budget, DP-ZOPO converges faster than FT in the early stage while FT can continue to converge more stably in the later stages.}
    \label{fig:enter-label}
    \vspace{1em}
\end{figure}

\begin{table}[htb]
\vspace{1em}
\caption{Experiments on RoBERTa-large (350M parameters). Dynamic pruning with a constant pruning rate fails to improve optimization.}
\centering
\label{table_roberta_cm&dm}
\begin{tabular}{c c c c c c}
\toprule
\specialrule{0em}{1pt}{1pt}
\multicolumn{1}{c}{Task} & SST-2 & SST-5 & SNLI & MNLI & TREC \\
\hline
\hline
\specialrule{0em}{1pt}{1pt}
\multicolumn{6}{c}{Small Privacy budget: $\epsilon=2$} \\
\hline
\specialrule{0em}{1pt}{1pt}
\multicolumn{1}{c}{\small static mask} & 92.5 & 48.8 & 79.5 & 71.4 & 89.6 \\
\multicolumn{1}{c}{\small dynamic mask} & 92.4 $\downarrow$ & 50.5 $\uparrow$ & 78.6 $\downarrow$ & 70.1 $\downarrow$ & 85.8 $\downarrow$ \\
\hline
\hline
\specialrule{0em}{1pt}{1pt}
\multicolumn{6}{c}{Medium Privacy budget: $\epsilon=4$} \\
\hline
\specialrule{0em}{1pt}{1pt}
\multicolumn{1}{c}{\small static mask} & 93.2 & 50.5 & 80.4 & 71.5 & 88.2 \\
\multicolumn{1}{c}{\small dynamic mask} & 93.1 $\downarrow$ & 49.9 $\downarrow$ & 78.5 $\downarrow$ & 69.6 $\downarrow$ & 87.8 $\downarrow$ \\
\hline
\hline
\specialrule{0em}{1pt}{1pt}
\multicolumn{6}{c}{Large Privacy budget: $\epsilon=8$} \\
\hline
\specialrule{0em}{1pt}{1pt}
\multicolumn{1}{c}{\small static mask} & 93.6 & 49.8 & 81.4 & 73.1 & 90.8 \\
\multicolumn{1}{c}{\small dynamic mask} & 92.5 $\downarrow$ & 50.3 $\uparrow$ & 79.3 $\downarrow$ & 69.5 $\downarrow$ & 87.6 $\downarrow$ \\
\bottomrule
\end{tabular}
\end{table}

\begin{table}[b]
\caption{Experiments on RoBERTa-large (350M parameters). $r=0.5\%$ denotes the static pruning with pruning rate $r=0.5\%$ before fine-tuning and $r=0.5\%$* denotes the dynamic pruning with constant pruning rate $r=0.5\%$.}
\centering
\label{roberta static vs dynamic}
\begin{tabular}{c c c c c c c}
\toprule
\specialrule{0em}{1pt}{1pt}
\multicolumn{2}{c}{Task} & SST-2 & SST-5 & SNLI & MNLI & TREC \\
\hline
\hline
\specialrule{0em}{1pt}{1pt}
\multicolumn{7}{c}{Small Privacy budget: $\epsilon=2$} \\
\hline
\specialrule{0em}{1pt}{1pt}
\multicolumn{2}{l}{$r=0.5\%$} & 88.5 & 46.5 & 70.3 & 70.1 & 83.4 \\
\multicolumn{2}{l}{$r=0.5\%$*} & 88.8 & 50.5 & 75.5 & 70.1 & 84.2 \\
\multicolumn{2}{l}{$r=1\%$} & 91.7 & 48.8 & 78.4 & 68.1 & 88.6 \\
\multicolumn{2}{l}{$r=1\%$*} & 91.7 & 48.8 & 76.9 & 64.3 & 84.0 \\
\multicolumn{2}{l}{$r=2\%$} & 88.4 & 47.1 & 78.2 & 71.4 & 89.6 \\
\multicolumn{2}{l}{$r=2\%$*} & 90.9 & 50.0 & 78.6 & 67.9 & 84.0 \\
\multicolumn{2}{l}{$r=4\%$} & 92.2 & 46.6 & 79.5 & 69.8 & 85.0 \\
\multicolumn{2}{l}{$r=4\%$*} & 92.1 & 48.9 & 76.9 & 69.6 & 85.8 \\
\multicolumn{2}{l}{$r=8\%$} & 92.5 & 46.6 & 77.3 & 69.3 & 85.4 \\
\multicolumn{2}{l}{$r=8\%$*} & 92.4 & 47.6 & 77.2 & 69.5 & 84.8 \\
\hline
\hline
\specialrule{0em}{1pt}{1pt}
\multicolumn{7}{c}{Medium Privacy budget: $\epsilon=4$} \\
\hline
\specialrule{0em}{1pt}{1pt}
\multicolumn{2}{l}{$r=0.5\%$} & 92.2 & 50.5 & 78.6 & 71.4 & 81.0 \\
\multicolumn{2}{l}{$r=0.5\%$*} & 92.1 & 48.8 & 78.2 & 69.0 & 87.8 \\
\multicolumn{2}{l}{$r=1\%$} & 92.2 & 47.1 & 80.4 & 64.7 & 88.2 \\
\multicolumn{2}{l}{$r=1\%$*} & 91.6 & 49.1 & 76.8 & 66.6 & 84.6 \\
\multicolumn{2}{l}{$r=2\%$} & 91.6 & 47.7 & 80.4 & 71.5 & 87.0 \\
\multicolumn{2}{l}{$r=2\%$*} & 92.7 & 49.9 & 78.5 & 68.8 & 84.6 \\
\multicolumn{2}{l}{$r=4\%$} & 93.2 & 47.0 & 78.3 & 69.7 & 86.4 \\
\multicolumn{2}{l}{$r=4\%$*} & 92.1 & 48.9 & 78.4 & 69.6 & 83.2 \\
\multicolumn{2}{l}{$r=8\%$} & 92.5 & 48.4 & 78.6 & 69.7 & 84.4 \\
\multicolumn{2}{l}{$r=8\%$*} & 93.1 & 49.3 & 77.9 & 69.0 & 83.2 \\
\hline
\hline
\specialrule{0em}{1pt}{1pt}
\multicolumn{7}{c}{Large Privacy budget: $\epsilon=8$} \\
\hline
\specialrule{0em}{1pt}{1pt}
\multicolumn{2}{l}{$r=0.5\%$} & 92.2 & 49.8 & 81.4 & 71.6 & 85.4 \\
\multicolumn{2}{l}{$r=0.5\%$*} & 92.5 & 48.8 & 76.2 & 67.7 & 87.6 \\
\multicolumn{2}{l}{$r=1\%$} & 92.2 & 47.4 & 79.2 & 73.1 & 90.8 \\
\multicolumn{2}{l}{$r=1\%$*} & 91.7 & 47.9 & 77.5 & 65.6 & 82.6 \\
\multicolumn{2}{l}{$r=2\%$} & 91.8 & 49.4 & 78.9 & 68.6 & 85.4 \\
\multicolumn{2}{l}{$r=2\%$*} & 92.5 & 48.6 & 79.3 & 67.7 & 82.4 \\
\multicolumn{2}{l}{$r=4\%$} & 93.6 & 48.0 & 77.4 & 68.4 & 87.0 \\
\multicolumn{2}{l}{$r=4\%$*} & 92.1 & 47.1 & 77.0 & 69.3 & 85.0 \\
\multicolumn{2}{l}{$r=8\%$} & 92.4 & 48.0 & 79.0 & 69.7 & 84.4 \\
\multicolumn{2}{l}{$r=8\%$*} & 92.4 & 46.5 & 77.2 & 69.5 & 87.0 \\
\bottomrule
\end{tabular}
\end{table}

\begin{table}[htb]
\caption{Experiments on RoBERTAa-large. 1-2-4\% $\ast$ and 1-2-4\% denote the incremental pruning and dynamic pruning strategy separately, both with dynamic pruning rate $r$ increasing from 1\% to 2\% and 4\%. Incremental pruning strategy outperforms dynamic pruning strategy in most dataset under all private settings.}
\centering
\label{table robert incremental}
\setlength{\tabcolsep}{2mm}{
\begin{tabular}{c c c c c c c}
\toprule
\multicolumn{2}{l}{Task} & $\textbf{SST-2}$ & $\textbf{SST-5}$ & $\textbf{SNLI}$ & $\textbf{MNLI}$ & $\textbf{TREC}$\\
\hline
\hline
\specialrule{0em}{1pt}{1pt}
\multicolumn{7}{c}{Small Privacy budget: $\epsilon=2$} \\
\hline
\specialrule{0em}{1pt}{1pt}
\multicolumn{2}{l}{0.5-1-2\%} & 89.4 & 48.7 & 79.2 & 71.3 & 86.6 \\
\multicolumn{2}{l}{0.5-1-2\% $\ast$} & 88.5 & 49.0 & 76.7 & 71.6 & 88.4 \\
\multicolumn{2}{l}{1-2-4\%} & 92.8 & 48.2 & 77.6 & 71.3 & 85.2 \\
\multicolumn{2}{l}{1-2-4\% $\ast$} & 92.2 & 47.7 & 80.1 & 70.7 & 91.2 \\
\multicolumn{2}{l}{2-4-8\%} & 89.6 & 46.4 & 75.3 & 72.0 & 91.4 \\
\multicolumn{2}{l}{2-4-8\% $\ast$} & 92.8 & 48.6 & 77.5 & 70.6 & 91.4 \\
\hline
\hline
\specialrule{0em}{1pt}{1pt}
\multicolumn{7}{c}{Medium Privacy budget: $\epsilon=4$} \\
\hline
\specialrule{0em}{1pt}{1pt}
\multicolumn{2}{l}{0.5-1-2\%} & 93.0 & 48.0 & 79.2 & 71.3 & 86.6 \\
\multicolumn{2}{l}{0.5-1-2\% $\ast$} & 93.6 & 49.1 & 80.0 & 73.1 & 90.4 \\
\multicolumn{2}{l}{1-2-4\%} & 92.8 & 48.2 & 79.4 & 71.2 & 87.4 \\
\multicolumn{2}{l}{1-2-4\% $\ast$} & 92.9 & 49.3 & 80.1 & 71.6 & 92.0 \\
\multicolumn{2}{l}{2-4-8\%} & 92.0 & 47.9 & 78.7 & 74.9 & 88.6 \\
\multicolumn{2}{l}{2-4-8\% $\ast$} & 92.8 & 48.6 & 80.6 & 73.9 & 88.6 \\
\hline
\hline
\specialrule{0em}{1pt}{1pt}
\multicolumn{7}{c}{Large Privacy budget: $\epsilon=8$} \\
\hline
\specialrule{0em}{1pt}{1pt}
\multicolumn{2}{l}{0.5-1-2\%} & 92.2 & 49.8 & 80.8 & 71.9 & 88.0 \\
\multicolumn{2}{l}{0.5-1-2\% $\ast$} & 93.6 & 48.5 & 79.2 & 74.8 & 89.2 \\
\multicolumn{2}{l}{1-2-4\%} & 92.5 & 48.8 & 79.3 & 72.0 & 89.0 \\
\multicolumn{2}{l}{1-2-4\% $\ast$} & 92.9 & 48.1 & 79.4 & 72.4 & 88.0 \\
\multicolumn{2}{l}{2-4-8\%} & 92.9 & 46.0 & 81.4 & 74.4 & 89.0 \\
\multicolumn{2}{l}{2-4-8\% $\ast$} & 92.8 & 48.6 & 81.6 & 74.8 & 91.8 \\
\bottomrule
\end{tabular}}
\end{table}

\begin{table}[htbp]
\vspace{1em}
\caption{Experiments on OPT-2.7b with privacy budget $\epsilon=4$. 1-2-4\% $\ast$ denotes the incremental pruning strategy of with $r=$1-2-4\%. Dynamic pruning outperforms static pruning in all five datasets.}
\centering
\label{optincre_vs}
\begin{tabular}{c c c c c c}
\toprule
\specialrule{0em}{1pt}{1pt}
\multicolumn{2}{c}{Task} & SST-2 & CB & BoolQ & MultiRC \\
\hline
\specialrule{0em}{1pt}{1pt}
\multicolumn{2}{l}{static*} & 90.0 & 62.0 & 63.7 & 50.0 \\
\multicolumn{2}{l}{$0.5-1-2\%$} & 92.5 & 76.0 & 67.5 & 65.0 \\
\multicolumn{2}{l}{$0.5-1-2\% \ast$} & 92.2 & 77.2 & 69.2 & 61.6 \\
\multicolumn{2}{l}{$1-2-4\%$} & 88.1 & 77.2 & 67.3 & 55.8 \\
\multicolumn{2}{l}{$1-2-4\% \ast$} & 92.3 & 78.8 & 67.3 & 58.3 \\
\multicolumn{2}{l}{$2-4-8\%$} & 90.0 & 76.8 & 66.7 & 54.0 \\
\multicolumn{2}{l}{$2-4-8\% \ast$} & 89.0 & 76.8 & 66.7 & 57.0 \\
\bottomrule
\end{tabular}
\end{table}

\begin{table}[htbp]
\caption{OPT-1.3b and OPT-2.7b fine-tuned on SQuAD. DP-ZOPO has better QA accuracy than ICL, DPZero, DP-ZOSO, FT-prefix, even FT in all private settings.}
\raggedright  
\label{table_squad}
\vspace{1em}
\begin{subtable}{\linewidth}
\centering
\caption{OPT-1.3b}
\setlength{\tabcolsep}{5mm}{
\begin{tabular}{c c c c c}
\toprule

\multicolumn{2}{l}{Method} & $\epsilon=2$ & $\epsilon=4$ & $\epsilon=8$ \\
\hline
\hline
\specialrule{0em}{1pt}{1pt}
\multicolumn{2}{l}{Zero-shot} & 27.23 & 27.23 & 27.23 \\
\multicolumn{2}{l}{ICL} & 58.66 & 58.66 & 58.66 \\
\multicolumn{2}{l}{FT} & 71.71 & 72.24 & 75.30 \\
\multicolumn{2}{l}{FT-prefix} & 71.37 & 72.80 & 73.21 \\
\multicolumn{2}{l}{DPZero} & 48.38 & 52.17 & 47.46 \\
\multicolumn{2}{l}{DP-ZOSO} & 73.31 & 73.01 & 75.24 \\
\multicolumn{2}{l}{DP-ZOPO} & 73.71 & 76.00 & 75.24 \\
\bottomrule
\end{tabular}}
\end{subtable}

\vspace{2em}

\begin{subtable}{\linewidth}
\centering
\caption{OPT-2.7b}
\setlength{\tabcolsep}{5mm}{
\begin{tabular}{c c c c c}
\toprule
\multicolumn{2}{l}{Method} & $\epsilon=2$ & $\epsilon=4$ & $\epsilon=8$ \\
\hline
\hline
\specialrule{0em}{1pt}{1pt}
\multicolumn{2}{l}{Zero-shot} & 29.89 & 29.89 & 29.89 \\
\multicolumn{2}{l}{ICL} & 67.63 & 67.63 & 67.63 \\
\multicolumn{2}{l}{FT} & 77.35 & 77.41 & 78.35 \\
\multicolumn{2}{l}{FT-prefix} & 76.50 & 76.66 & 76.92 \\
\multicolumn{2}{l}{DPZero} & 60.18 & 63.46 & 65.12 \\
\multicolumn{2}{l}{DP-ZOSO} & 76.83 & 77.68 & 78.16 \\
\multicolumn{2}{l}{DP-ZOPO} & 78.83 & 78.03 & 79.03 \\
\bottomrule
\end{tabular}}
\end{subtable}
\end{table}



\end{document}